\documentclass{article}

\PassOptionsToPackage{round}{natbib}

\usepackage[final]{neurips_2022}

\usepackage{silence}
\usepackage[utf8]{inputenc} 
\usepackage[T1]{fontenc}    
\usepackage{url}            
\usepackage{booktabs}       
\usepackage{amsfonts}       
\usepackage{nicefrac}       
\usepackage{microtype}      
\usepackage{xcolor}         

\usepackage{amsmath}
\usepackage{amssymb}
\usepackage{mathtools}
\usepackage{amsthm}
\usepackage[]{amsbsy}
\usepackage{commath}
\usepackage{thm-restate} 
\usepackage{xspace} 
\usepackage{algorithm}
\usepackage[noend]{algorithmic}

\usepackage{caption}
\usepackage{subcaption}
\usepackage{placeins}
\usepackage{bm}
\usepackage{comment}

\usepackage{siunitx}
\let\svqty\qty
\usepackage{physics}
\let\qty\svqty

\usepackage{ctable}

\usepackage[hidelinks]{hyperref}       
\usepackage[capitalize,noabbrev]{cleveref}
\crefname{equation}{}{}

\theoremstyle{plain}

\newtheorem{lemma}{Lemma}
\newtheorem{corollary}{Corollary}

\theoremstyle{definition}

\theoremstyle{remark}

\crefname{assumption}{assumption}{assumptions}
\crefname{equation}{}{}

\renewcommand{\paragraph}[1]{\textbf{#1}\hspace{1em}}

\newcommand{\ourmethod}{\textsc{ISE}\xspace}
\newcommand{\ourmethodlong}{Information-Theoretic Safe Exploration\xspace}
\newcommand{\safeopt}{\textsc{SafeOpt}\xspace}
\newcommand{\stageopt}{\textsc{StageOpt}\xspace}
\newcommand{\linebo}{\textsc{LineBO}\xspace}

\DeclareMathOperator*{\argmax}{arg\,max}
\newcommand{\cpsi}[0]{\mathbb{I}_{f(\bm x) \geq 0}}

\let\originalleft\left
\let\originalright\right
\renewcommand{\left}{\mathopen{}\mathclose\bgroup\originalleft}
\renewcommand{\right}{\aftergroup\egroup\originalright}

\title{Information-Theoretic Safe Exploration with\\ Gaussian Processes}

\author{%
  Alessandro G. Bottero$^{1,2}$, Carlos E. Luis$^{1,2}$, Julia Vinogradska$^1$, Felix Berkenkamp$^1$, Jan Peters$^2$\\
$^1$Bosch Center for Artificial Intelligence, Germany\\
$^2$Technische Universität Darmstadt, Germany \\
\texttt{AlessandroGiacomo.Bottero@de.bosch.com}
}

\begin{document}

\maketitle

\begin{abstract}
We consider a sequential decision making task where we are not allowed to evaluate parameters that violate an \textit{a~priori} unknown (safety) constraint. A common approach is to place a Gaussian process prior on the unknown constraint and allow evaluations only in regions that are safe with high probability. Most current methods rely on a discretization of the domain and cannot be directly extended to the continuous case. Moreover, the way in which they exploit regularity assumptions about the constraint introduces an additional critical hyperparameter. In this paper, we propose an information-theoretic safe exploration criterion that directly exploits the GP posterior to identify the most informative safe parameters to evaluate. Our approach is naturally applicable to continuous domains and does not require additional hyperparameters. We theoretically analyze the method and show that we do not violate the safety constraint with high probability and that we explore by learning about the constraint up to arbitrary precision. Empirical evaluations demonstrate improved data-efficiency and scalability.
\end{abstract}

\section{Introduction}
\label{section:introduction}

In sequential decision making problems, we iteratively select parameters in order to optimize a given performance criterion. However, real-world applications such as robotics \citep{berkenkamp_bayesian_2020}, mechanical systems \citep{combustion_engine_bo} or medicine \citep{sui_safe_2015} are often subject to additional safety constraints that we cannot violate during the exploration process \citep{dulac-arnold_challenges_nodate}. Since it is \textit{a priori} unknown which parameters lead to constraint violations, we need to actively and carefully learn about the constraints without violating them. That is, we need to learn about the safety of parameters by only evaluating parameters that are currently known to be safe.

Existing methods by \citet{Schreiter_safe_exp_2015,sui_safe_2015} tackle this problem by placing a Gaussian process (GP) prior over the constraint and only evaluate parameters that do not violate the constraint with high probability. To learn about the safety of parameters, they evaluate the parameter with the largest posterior variance. This process is made more efficient by \safeopt, which restricts its safe set expansion exploration component to parameters that are close to the boundary of the current set of safe parameters \citep{sui_safe_2015} at the cost of an additional tuning hyperparameter (Lipschitz constant). However, uncertainty about the constraint is only a proxy objective that only indirectly learns about the safety of parameters. Consequently, data-efficiency could be improved with an exploration criterion that directly maximizes the information gained about the safety of parameters.

\paragraph{Our contribution}
In this paper, we propose \ourmethodlong (\ourmethod), a safe exploration algorithm that \emph{directly} exploits the information gain about the safety of parameters in order to expand the region of the parameter space that we can classify as safe with high confidence. By directly optimizing for safe information gain, \ourmethod is more data-efficient than existing approaches without manually restricting evaluated parameters to be on the boundary of the safe set, particularly in scenarios where the posterior variance alone is not enough to identify good evaluation candidates, as in the case of heteroskedastic observation noise. This exploration criterion also means that we do not require additional hyperparameters beyond the GP posterior and that \ourmethod is directly applicable to continuous domains. We theoretically analyze our method and prove that it learns about the safety of reachable parameters to arbitrary precision.

\paragraph{Related work} Information-based selection criteria with Gaussian processes models are successfully used in the context of unconstrained Bayesian optimization (BO, \citet{shahriari_taking_2016,DBLP:journals/corr/abs-1204-5721}), where the goal is to find the parameters that maximize an \textit{a priori} unknown function.  \citet{HennigS2012,hernandez-lobato_predictive_2014,wang_max-value_2017} select parameters that provide the most information about the optimal parameters, while \citet{frohlich_noisy-input_2020} consider the information under noisy parameters. The success of these information-based approaches also relies on the superior data efficiency that they demonstrated. We draw inspiration from these methods when defining an information-based criterion w.r.t.\@ the safety of parameters to guide safe exploration.

In the presence of constraints that the final solution needs to satisfy, but which we can violate during exploration, \citet{constrained_bo} propose to combine typical BO acquisition functions with the probability of satisfying the constraint. Instead, \citet{lse_krause} propose an uncertainty-based criterion that learns about the feasible region of parameters. 
When we are not allowed to ever evaluate unsafe parameters, safe exploration is a necessary sub-routine of BO algorithms to learn about the safety of parameters. To safely explore, \citet{Schreiter_safe_exp_2015} globally learn about the constraint by evaluating the most uncertain parameters. \safeopt by \citet{sui_safe_2015} extends this to joint exploration and optimization and makes it more efficient by explicitly restricting safe exploration to the boundary of the safe set. \citet{sui2018stagewise} proposes \stageopt, which additionally separates the exploration and optimization phases. Both of these algorithms assume access to a Lipschitz constant to define parameters close to the boundary of the safe set, which is a difficult tuning parameter in practice. These methods have been extended to multiple constraints by \citet{berkenkamp_bayesian_2020}, while \citet{kirschner_adaptive_2019} scale them to higher dimensions with \linebo, which explores in low-dimensional sub-spaces. To improve computational costs, \citet{duivenvoorden_constrained_2017} suggest a continuous approximation to \safeopt without providing exploration guarantees. 
All of these methods rely on function uncertainty to drive exploration, while we directly maximize the information gained about the safety of parameters. 

Safe exploration also arises in the context of Markov decision processes (MDP), \citep{safe_mdp_ergodicity,safe_mdp_zoo}. In particular, \citet{turchetta_mdp_2016,turchetta_krause_2019} traverse the MDP to learn about the safety of parameters using methods that, at their core, explore using the same ideas as \safeopt and \stageopt to select parameters to evaluate. Consequently, our proposed method for safe exploration is also directly applicable to their setting.

\section{Problem Statement}\label{sec:problem_statement}

In this section, we introduce the problem and notation that we use throughout the paper.
We are given an unknown and expensive to evaluate safety constraint $f: \mathcal{X} \rightarrow \mathbb{R}$ s.t.\@ parameters that satisfy $f(\bm x) \geq 0$ are classified as safe, while others are unsafe. To start exploring safely, we also have access to an initial safe parameter $\bm x_0$ that satisfies the safety constraint, $f(\bm x_0) \geq 0$.
We sequentially select safe parameters $\bm x_n \in \mathcal{X}$ where to evaluate $f$ in order to learn about the safety of parameters beyond $\bm x_0$. At each iteration $n$, we obtain a noisy observation of $y_n \coloneqq f(\bm x_n) + \nu_n$ that is corrupted by additive homoscedastic Gaussian noise $\nu_n \sim \mathcal{N}\left(0, \sigma^2_\nu\right)$. We illustrate the task in \cref{fig:problem_statement_example}, where starting from $\bm x_0$ we aim to safely explore the domain so that we ultimately classify as safe all the safe parameters that are reachable from $\bm x_0$.

\begin{figure}[t]
\captionsetup[subfigure]{position=b}
\centering
\subcaptionbox{Problem components.\label{fig:problem_statement_example}}{\includegraphics[height=0.15\textheight]{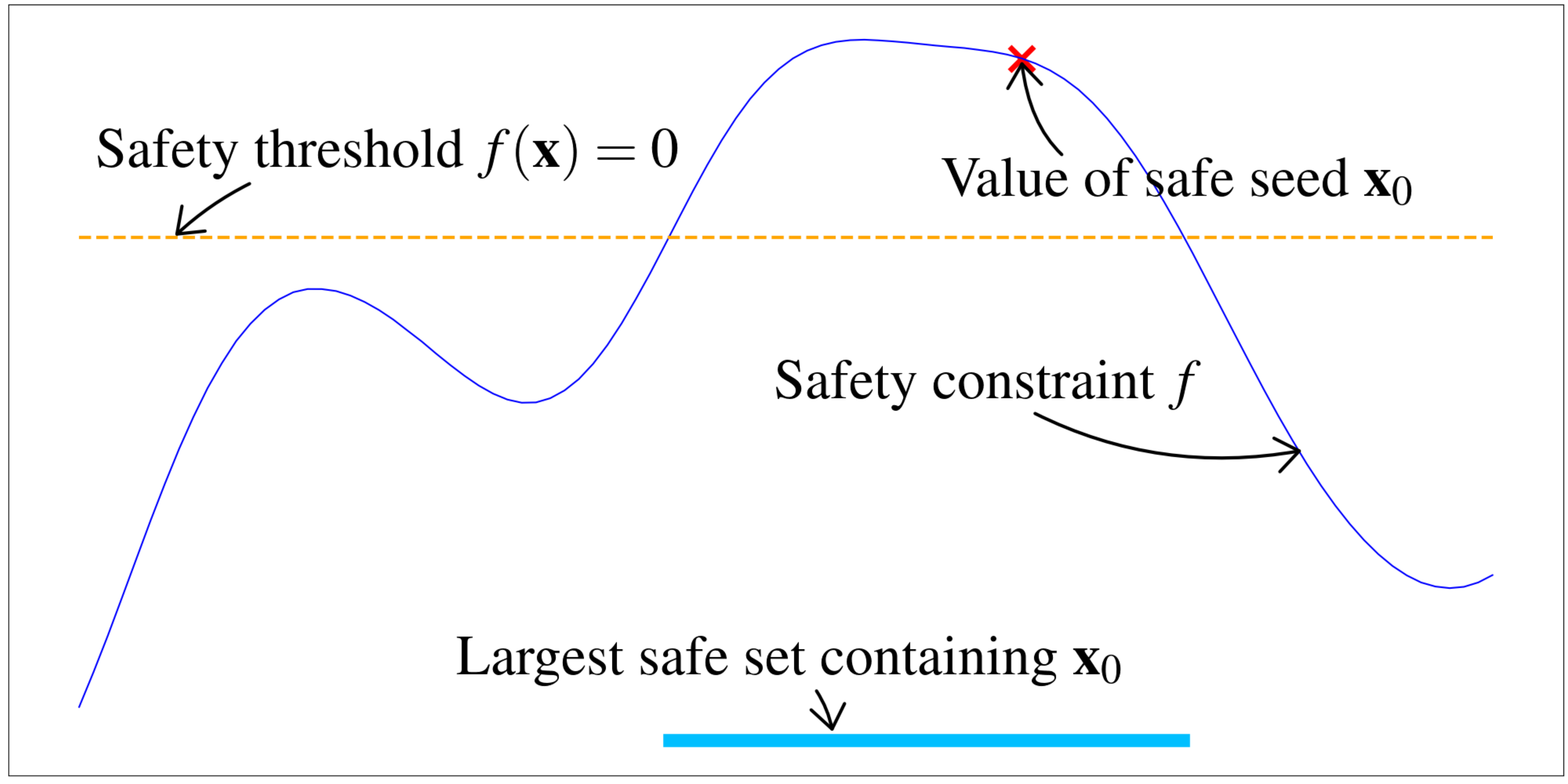}}
\hfill
\subcaptionbox{\ourmethod mutual information.\label{fig:mutual_information}}{\includegraphics[height=0.15\textheight]{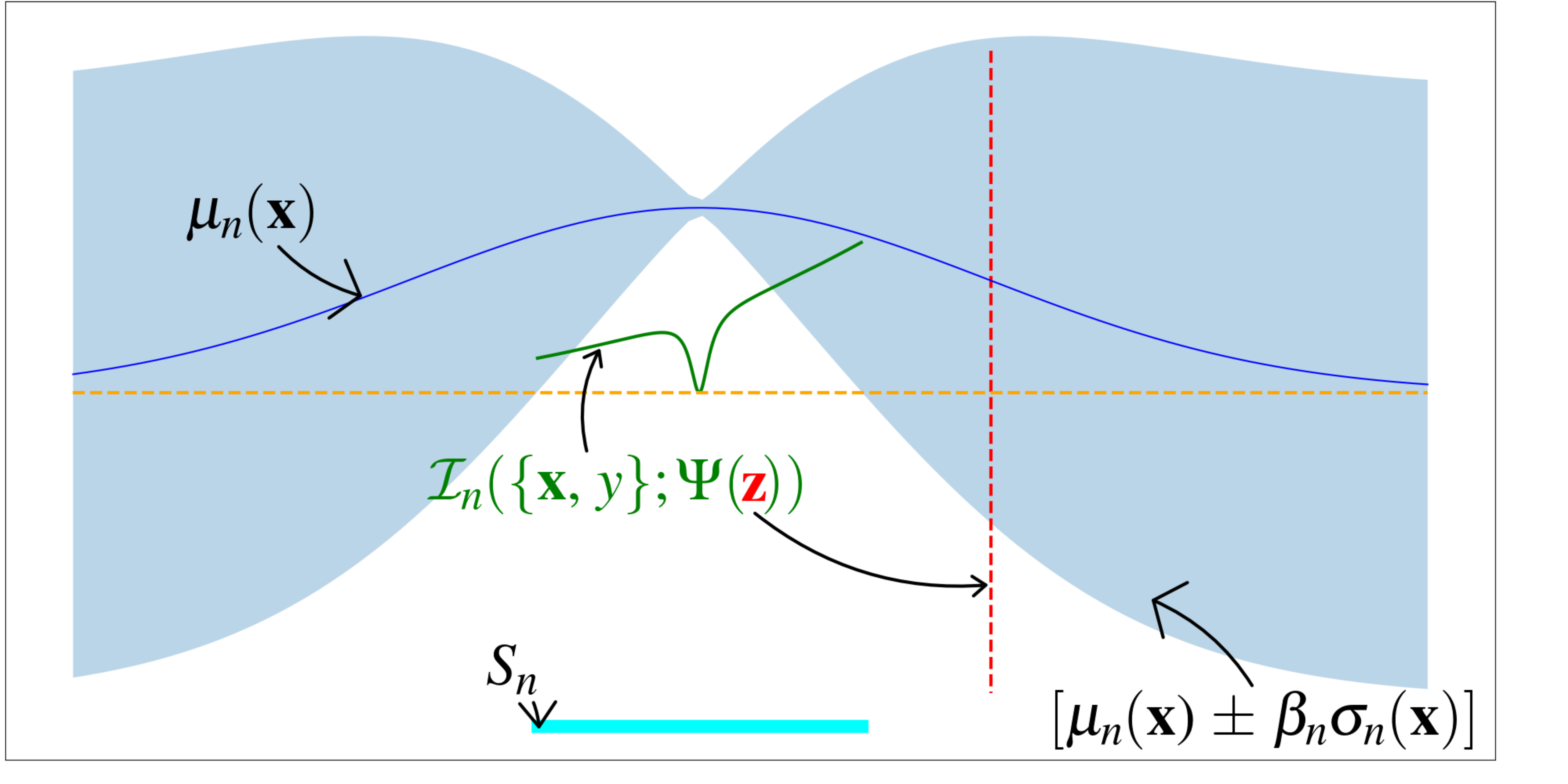}}
\caption{In (\subref{fig:problem_statement_example}) we illustrate the safe exploration task. Based on the unknown safety constraint $f$, we are only allowed to evaluate safe parameters $\bm x$ with values $f(\bm x) \geq 0$ above the safety threshold (dashed line). Starting from a safe seed $\bm x_0$ a safe exploration strategy needs to discover the largest reachable safe region of the parameter space containing $\bm x_0$. In (\subref{fig:mutual_information}) we show the mutual information $I_n(\{\bm x, y\}; \Psi(\bm z))$ in green for different $\bm x$ inside the safe set and for a fixed $\bm z$ outside (red dashed line). \ourmethod maximizes this mutual information jointly over $\bm x$ and $\bm z$. \label{fig:acquisition_representation}}
\hfill
\end{figure}

As $f$ is unknown and the evaluations $y_n$ are noisy, it is not feasible to select parameters that are safe with certainty and we provide high-probability safety guarantees instead. To this end, we assume that the safety constraint $f$ has bounded norm in the Reproducing Kernel Hilbert Space (RKHS) \citep{Scholkopf2002} $\mathcal{H}_k$ associated to some kernel $k: \mathcal{X} \times \mathcal{X} \rightarrow \mathbb{R}$ with $k(\bm x, \bm x') \leq 1$. This assumption allows us to to model $f$ as a Gaussian process (GP) \citep{srinivas_gaussian_2010}.

\paragraph{Gaussian Processes}
A GP is a stochastic process specified by a mean function $\mu: \mathcal{X}\rightarrow \mathbb{R}$ and a kernel $k$ \citep{rassmussen_gaussian_2006}. It defines a probability distribution over real-valued functions on $\mathcal{X}$, such that any finite collection of function values at parameters $[\bm x_1, \dots, \bm x_n]$ is distributed as a multivariate normal distribution. 
The GP prior can then be conditioned on (noisy) function evaluations $\mathcal{D}_n = \{(\bm x_i, y_i)\}_{i=1}^n$. If the noise is Gaussian, then the resulting posterior is also a GP and with posterior mean and variance
\begin{equation}
\begin{split}
\mu_n(\bm x) &= \mu(\bm x) + \bm k(\bm x)^\top (\bm K + \bm I \sigma_\nu^2)^{-1}\left(\bm y - \bm \mu \right), \\
\sigma_n^2(\bm x)
&= k(\bm x, \bm x) - \bm k(\bm x)^\top (\bm K + \bm I \sigma_\nu^2)^{-1} \bm k(\bm x),
\end{split}
\end{equation}
where $\bm \mu \coloneqq [\mu(\bm x_1), \dots \mu(\bm x_n)]$ is the mean vector at parameters $\bm x_i \in \mathcal{D}_n$ and $\left[\bm y\right]_i \coloneqq y(\bm x_i)$ the corresponding vector of observations. We have $\left[\bm k(\bm x)\right]_i \coloneqq k(\bm x, \bm x_i)$, the kernel matrix has entries $\left[\bm K\right]_{ij} \coloneqq k(\bm x_i, \bm x_j)$, and $\bm I$ is the identity matrix. In the following, we assume without loss of generality that the prior mean is identically zero: $\mu(\bm x) \equiv 0$.

\paragraph{Safe set}
Using the previous assumptions, we can construct high-probability confidence intervals on the function values $f(\bm x)$. Concretely, for any $\delta > 0$ it is possible to find a sequence of positive numbers $\{\beta_n\}$ such that $f(\bm x) \in \left[\mu_n(\bm x) \pm \beta_n\sigma_n(\bm x)\right]$ with probability at least $1 - \delta$, jointly for all $\bm x \in \mathcal{X}$ and $n \geq 1$. For a proof and more details see \citep{kernelized_bandits}. We use these confidence intervals to define a \textit{safe set}
\begin{equation}\label{eq:safe_set_definition}
S_n \coloneqq \{\bm x \in \mathcal{X} : \mu_n(\bm x) - \beta_n\sigma_n(\bm x) \geq 0\} \cup \{ \bm x_0 \} ,
\end{equation}
which contains all parameters whose $\beta_n$-lower confidence bound is above the safety threshold and the initial safe parameter $\bm x_0$. Consequently, we know that all parameters in $S_n$ are safe, $f(\bm x) \geq 0$ for all $\bm x \in S_n$, with probability at least $1 - \delta$ jointly over all iterations $n$.

\paragraph{Safe exploration}
Given the safe set $S_n$, the next question is which parameters in $S_n$ to evaluate in order to efficiently expand it. Most existing safe exploration methods rely on uncertainty sampling over subsets of $S_n$. \safeopt-like approaches, for example, use the Lipschitz assumption on $f$ to identify parameters in $S_n$ that could expand the safe set and then select the parameter that has the biggest uncertainty among those. In the next sections, we present and analyze our safe exploration strategy, \ourmethod, that instead uses an information gain measure to identify the parameters that allow us to efficiently learn about the safety of parameters outside of $S_n$. 

\section{\ourmethodlong}\label{algorithm}
We present \ourmethodlong (\ourmethod), which guides the safe exploration by using an information-theoretic criterion.
Our goal is to design an exploration strategy that directly exploits the properties of GPs to learn about the safety of parameters outside of $S_n$. We draw inspiration from \citet{HennigS2012,wang_max-value_2017} who exploit information-theoretic insights to design data-efficient BO acquisition functions for their respective optimization objectives.

\paragraph{Information gain measure}
In our case, we want to evaluate $f$ at safe parameters that are maximally informative about the safety of other parameters, in particular of those where we are uncertain about whether they are safe or not. To this end, we need a corresponding measure of information gain. We define such a measure using the binary variable $\Psi(\bm x) = \cpsi$, which is equal to one iff $f(\bm x) \geq 0$. Its entropy is given by
\begin{equation}\label{eq:exact_psi_entropy}
H_n\left[\Psi(\bm z)\right] = -p_n^-(\bm z) \ln(p_n^-(\bm z)) - \left(1 - p_n^-(\bm z)\right) \ln(1 - p_n^-(\bm z))
\end{equation}
where $p_n^-(\bm z)$ is the probability of $\bm z$ being unsafe: $p_n^-(\bm z) = \frac{1}{2} + \frac{1}{2}\erf \left(-\frac{1}{\sqrt{2}}\frac{\mu_n(\bm z)}{\sigma_n(\bm z)}\right)$. The random variable $\Psi(\bm z)$ has high-entropy when we are uncertain whether a parameter is safe or not; that is, its entropy decreases monotonically as $|\mu_n(\bm z)|$ increases and the GP posterior moves away from the safety threshold. It also decreases monotonically as $\sigma_n(\bm z)$ decreases and we become more certain about the constraint. This behavior also implies that the entropy goes to zero as the confidence about the safety of $\bm z$ increases, as desired.

\begin{algorithm}[tb]
   \caption{\ourmethodlong}
   \label{algorithm:our_algorithm}
\begin{algorithmic}[1]
   \STATE {\bfseries Input:} GP prior ($\mu_0$, $k$, $\sigma_\nu$), Safe seed $\bm x_0$
   \FOR{$n=0$, \dots, $N$}
   \STATE $x_{n+1}$ $\leftarrow$ $\argmax_{\bm x \in S_n} \, \max_{\bm z \in \mathcal{X}}\hat{I}_n\left(\{\bm x, y\}; \Psi(\bm z)\right) $
   \STATE $y_{n+1}$ $\leftarrow$ $f(\bm x_{n+1}) + \nu$
   \STATE Update GP posterior with $(\bm x_{n+1}, y_{n+1})$ 
   \ENDFOR
\end{algorithmic}
\end{algorithm}

Given our definition of $\Psi$, we consider the mutual information $I\left(\{\bm x, y\}; \Psi(\bm z)\right)$ between an observation $y$ at a parameter $\bm x$ and the value of $\Psi$ at another parameter $\bm z$. Since $\Psi$ is the indicator function of the safe regions of the parameter space, the quantity $I_n\left(\{\bm x, y\}; \Psi(\bm z)\right)$ measures how much information about the safety of $\bm z$ we gain by evaluating the safety constraint $f$ at $\bm x$ at iteration $n$, averaged over all possible observed values $y$. This interpretation follows directly from the definition of mutual information: $I_n\left(\{\bm x, y\}; \Psi(\bm z)\right) = H_n\left[\Psi(\bm z)\right] - \mathbb{E}_{y}\left[H_{n+1}\left[\Psi(\bm z) \middle| \{\bm x, y\}\right]\right]$, where $H_n[\Psi(\bm z)]$ is the entropy of $\Psi(\bm z)$ according to the GP posterior at iteration $n$, while $H_{n+1}\left[\Psi(\bm z) \middle| \{\bm x, y\}\right]$ is its entropy at iteration $n+1$, conditioned on a measurement $y$ at $\bm x$ at iteration $n$. Intuitively, $I_n\left(\{\bm x, y\}; \Psi(\bm z)\right)$ is negligible whenever the confidence about the safety of $\bm z$ is high or, more generally, whenever an evaluation at $\bm x$ does not have the potential to substantially change our belief about the safety of $\bm z$. The mutual information is large whenever an evaluation at $\bm x$ on average causes the confidence about the safety of $\bm z$ to increase significantly. As an example, in \cref{fig:acquisition_representation} we plot $I_n\left(\{\bm x, y\}; \Psi(\bm z)\right)$ as a function of $\bm x \in S_n$ for a specific choice of $\bm z$ and for an RBF kernel. As one would expect, we see that the closer it gets to $\bm z$, the bigger the mutual information becomes, and that it vanishes in the neighborhood of previously evaluated parameters, where the posterior variance is negligible.

To compute $I_n\left(\{\bm x, y\}; \Psi(\bm z)\right)$, we need to average \cref{eq:exact_psi_entropy} conditioned on an evaluation $y$ over all possible values of $y$. However, the resulting integral is intractable given the expression of $H_n[\Psi(\bm z)]$ in \cref{eq:exact_psi_entropy}. In order to get a tractable result, we derive a close approximation of \cref{eq:exact_psi_entropy},
\begin{equation}\label{eq:approximated_psi_entropy}
H_n\left[\Psi(\bm z)\right] \approx \hat{H}_n\left[\Psi(\bm z)\right] \doteq \ln(2) \exp\left\{-\frac{1}{\pi\ln(2)}\left(\frac{\mu_n(\bm z)}{\sigma_n(\bm z)}\right)^2\right\}.
\end{equation}
The approximation in \cref{eq:approximated_psi_entropy} is obtained by truncating the Taylor expansion of $H_n[\Psi(\bm z)]$ at the second order, and it recovers almost exactly its true behavior (see \cref{appendix_entropy_approx} for details).
Since the posterior mean at $\bm z$ after an evaluation at $\bm x$ depends linearly on $\mu_n(\bm x)$, and since the probability density of $y$ depends exponentially on $-\mu_n^2(\bm x)$, using \cref{eq:approximated_psi_entropy} reduces the conditional entropy $\mathbb{E}_{y}\left[\hat{H}_{n+1}\left[\Psi(\bm z) \middle| \{\bm x, y\}\right]\right]$ to a Gaussian integral with the exact solution
\begin{equation}\label{eq:averaged_post_measurement_entropy}
\begin{split}
\mathbb{E}_{y}&\left[\hat{H}_{n+1}\left[\Psi(\bm z) \middle| \{\bm x, y\}\right]\right] =\\
&\ln(2)\sqrt{\frac{\sigma_\nu^2 + \sigma_n^2(\bm x)(1 - \rho_n^2(\bm x, \bm z))}{\sigma_\nu^2 + \sigma_n^2(\bm x)(1 + c_2\rho_n^2(\bm x, \bm z))}}\exp\left\{-c_1\frac{\mu_n^2(\bm z)}{\sigma_n^2(\bm z)}\frac{\sigma_\nu^2 + \sigma_n^2(\bm x)}{\sigma_\nu^2 + \sigma_n^2(\bm x)(1 + c_2\rho_n^2(\bm x, \bm z))}\right\},
\end{split}
\end{equation}
where $\rho_n(\bm x, \bm z)$ is the linear correlation coefficient between $f(\bm x)$ and $f(\bm z)$, and with $c_1$ and $c_2$ given by $c_1 \coloneqq 1/\ln(2)\pi$ and $c_2 \coloneqq 2c_1 - 1$. This result allows us to analytically calculate the approximated mutual information $\hat{I}_n\left(\{\bm x, y\}; \Psi(\bm z)\right) \doteq \hat{H}_n\left[\Psi(\bm z)\right] - \mathbb{E}_{y}\left[\hat{H}_{n+1}\left[\Psi(\bm z) \middle| \{\bm x, y\}\right]\right]$, which we use to define the \ourmethod acquisition function, and which we analyze theoretically in \cref{sec:theory}.

\paragraph{\ourmethod acquisition function}
Now that we have defined a way to measure and compute the information gain about the safety of parameters, we can use it to design an exploration strategy that selects the next parameters to evaluate. The natural choice for such selection criterion is to select the parameter that maximizes the information gain; that is, we select $\bm x_{n+1}$ according to

\begin{equation}\label{eq:x_n_+_1}
\bm x_{n+1} \in \argmax_{\bm x \in S_n} \, \max_{\bm z \in \mathcal{X}}\hat{I}_n\left(\{\bm x, y\}; \Psi(\bm z)\right) ,
\end{equation}
where we jointly optimize over $\bm x$ in the safe set $S_n$ and an unconstrained second parameter $\bm z$.
Evaluating $f$ at $\bm x_{n+1}$ according to \cref{eq:x_n_+_1} maximizes the information gained about the safety of some parameter $\bm z \in \mathcal{X}$, so that it allows us to efficiently learn about parameters that are not yet known to be safe. While $\bm z$ can lie in the whole domain, the parameters where we are the most uncertain about the safety constraint lie outside the safe set. By leaving $\bm z$ unconstrained, we show in our theoretical analysis in \cref{sec:theory} that, once we have learned about the safety of parameters outside the safe set, \cref{eq:x_n_+_1} resorts to learning about the constraint function also inside $S_n$. An overview of \ourmethod can be found in \cref{algorithm:our_algorithm} and we show an example run of a one-dimensional illustration of the algorithm in \cref{fig:algorithm_example}.

\begin{figure*}[!t!]
\hfill
     \centering
     \begin{subfigure}[b]{0.32\textwidth}
         \centering
         \includegraphics[width=\textwidth]{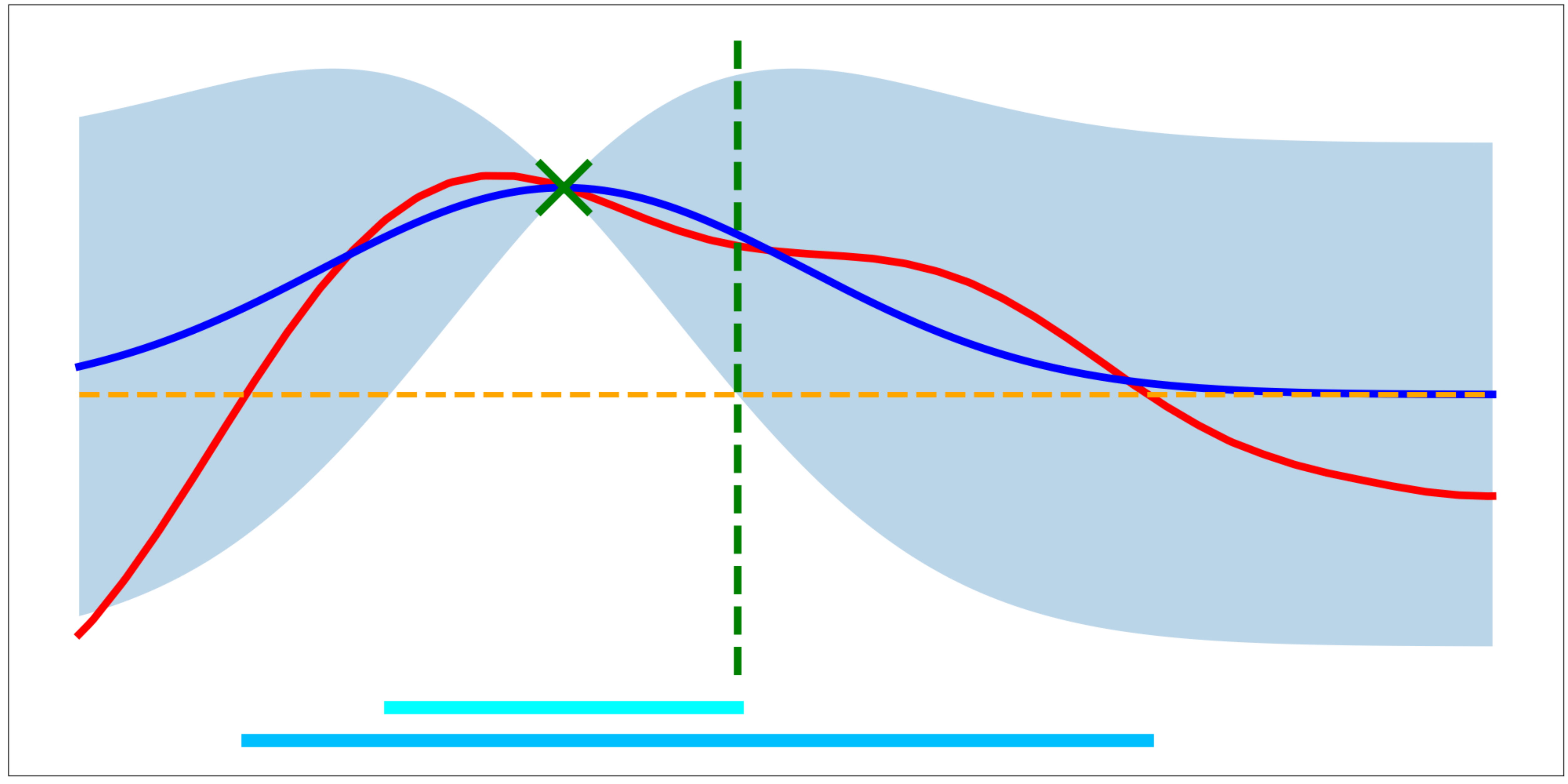}
         \caption{Iteration 1}
         \label{fig:algo_step_1}
     \end{subfigure}
     \hfill
     \begin{subfigure}[b]{0.32\textwidth}
         \centering
         \includegraphics[width=\textwidth]{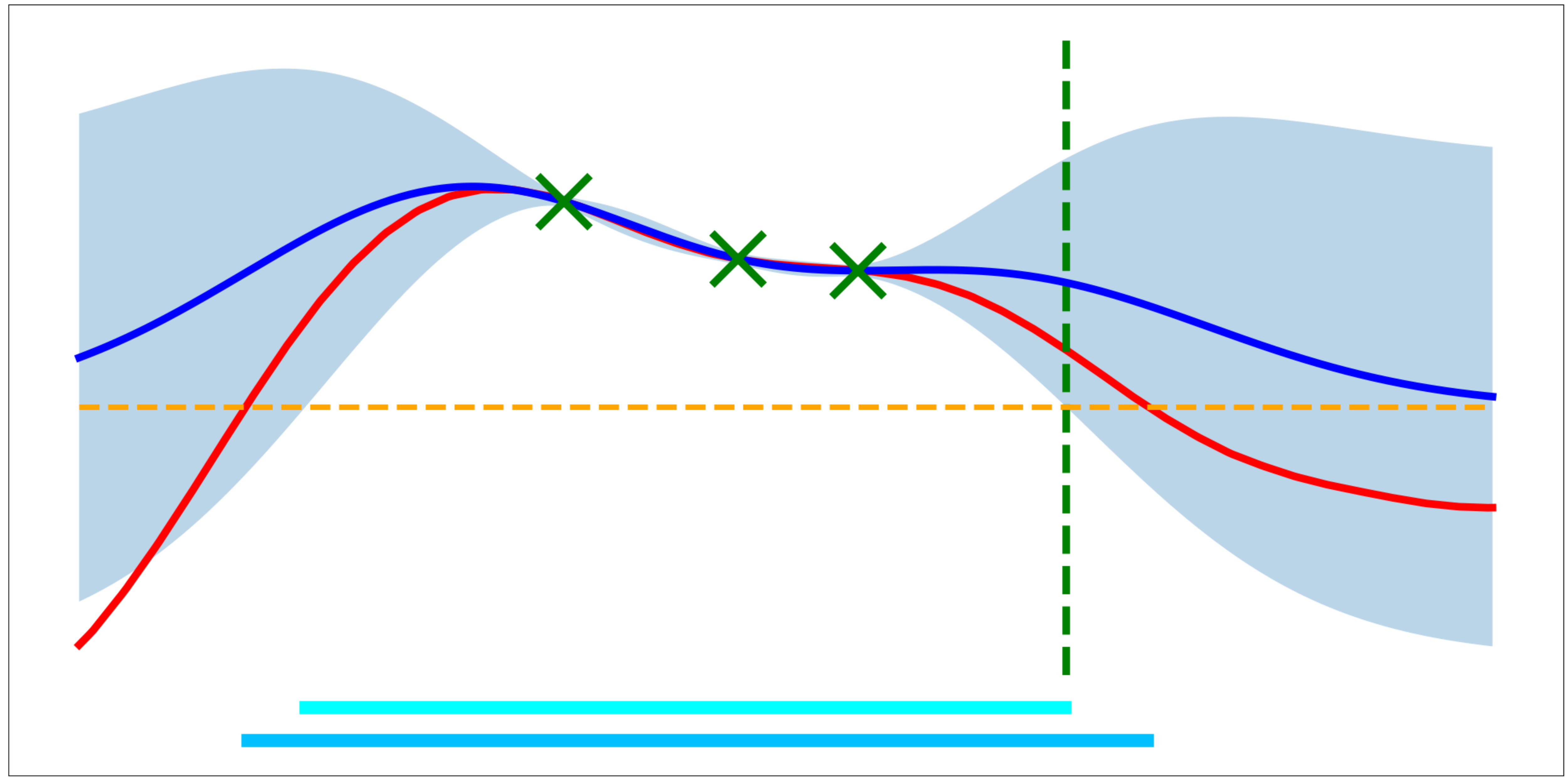}
         \caption{Iteration 3}
         \label{fig:algo_step_3}
     \end{subfigure}
     \hfill
     \begin{subfigure}[b]{0.32\textwidth}
         \centering
         \includegraphics[width=\textwidth]{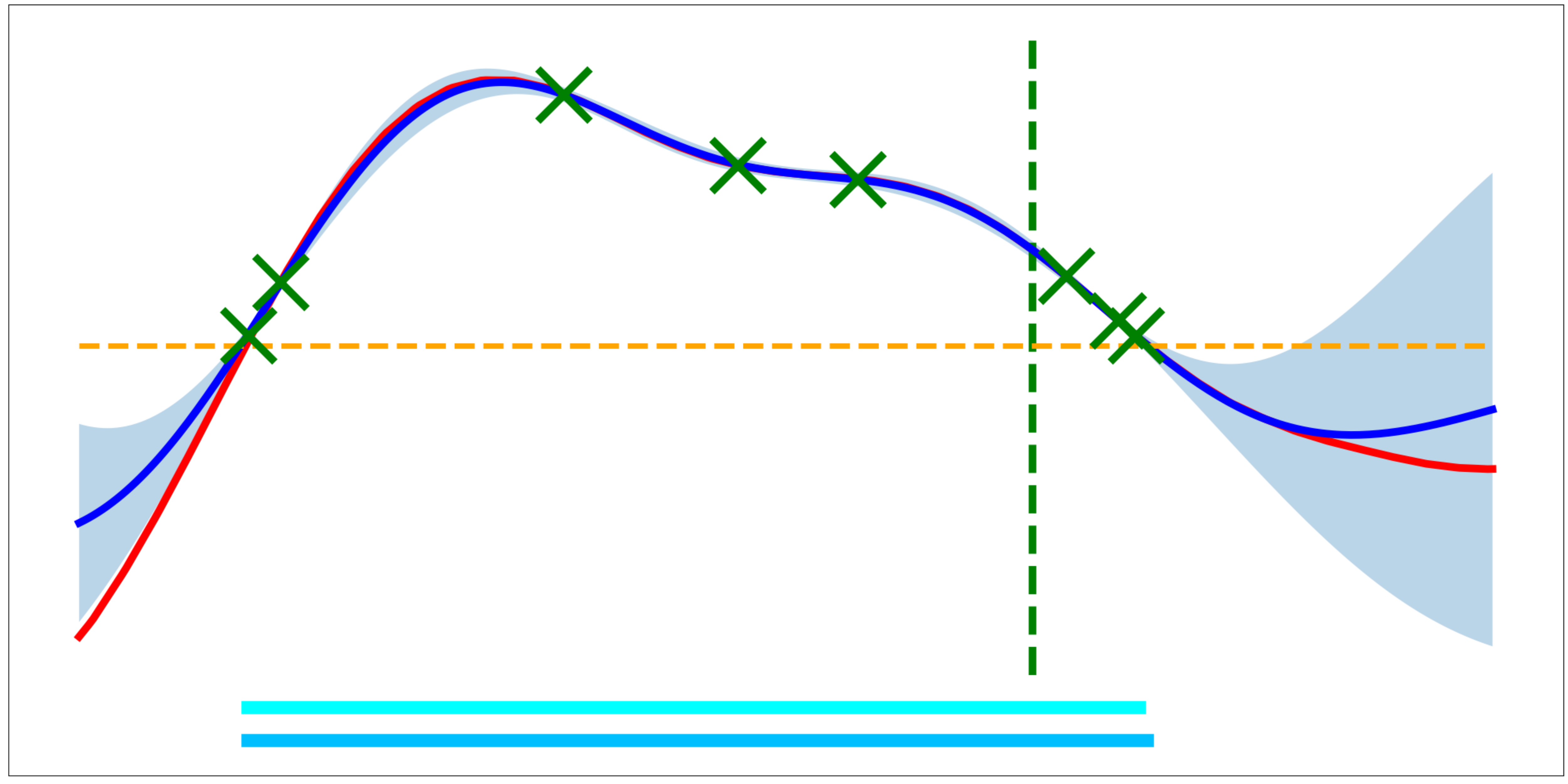}
         \caption{Iteration 8}
         \label{fig:algo_step_9}
     \end{subfigure}
        \caption{Example run of \ourmethod at different iterations. The GP's posterior mean (blue line) and confidence interval $\mu_n \pm \beta_n\sigma_n$ (blue shaded) approximate the true safety constraint $f$ (red line) based on selected data points (green crosses) and, together with the safety threshold at zero (orange line), identify the current safe set $S_n$ (light blue bar, bottom). The vertical green line indicates the location of the next parameter $\bm x_{n+1}$ selected by \cref{eq:x_n_+_1}. Initially, \ourmethod evaluates parameters on the boundary of the safe set, but eventually also evaluates inside the safe set if that provides the most information. \ourmethod quickly discover the largest reachable safe set (dark blue bar, bottom).}
        \label{fig:algorithm_example}
        \hfill
\end{figure*}

\FloatBarrier
\section{Theoretical Results}\label{sec:theory}

In this section, we study the expression for $\hat{I}_n\left(\{\bm x, y\}; \Psi(\bm z)\right)$ obtained using \cref{eq:approximated_psi_entropy,eq:averaged_post_measurement_entropy} and analyze the properties of the \ourmethod exploration criterion \cref{eq:x_n_+_1}. By construction of $S_n$ in \cref{eq:safe_set_definition} and the assumptions on $f$ in \cref{sec:problem_statement}, we know that any parameter selected according to \cref{eq:x_n_+_1} is safe with high probability,  see \cref{appendix_proofs} for details. To show that we also learn about the safe set, we first need to define what it means to successfully explore starting from $\bm x_0$. The main challenge is that it is difficult to analyze how a GP generalizes based on noisy observations, so that it is difficult to define a notion of convergence that is not dependent on the specific run. \safeopt avoids this issue by expanding the safe set not based on the GP, but only using the Lipschitz constant $L$. Contrary to their approach, we depend on the GP to generalize from the safe set. In this case, the natural notion of convergence is provided by the the posterior variance. In particular, we say that at iteration $n$ we have explored the safe set up to $\varepsilon$-accuracy if $\sigma_n^2(\bm x) \leq \varepsilon$ for all parameters $\bm x$ in $S_n$. In the following, we show that \ourmethod asymptotically leads either to $\varepsilon$-accurate exploration of the safe set or to indefinite expansion of the safe set. In future work it will be interesting to further investigate the notion of generalization and to derive a similar convergence result as those obtained by \cite{sui_safe_2015}.

\begin{restatable}{theorem}{convergenceboundafterexpansion}
\label{lemma:convergence_bound_after_expansion}
Assume that $\bm{x}_{n+1}$ is chosen according to \cref{eq:x_n_+_1}, and that there exists $\hat{n}$ such that for all $n \geq \hat{n}$ $S_{n+1} \subseteq S_n$. Moreover, assume that for all $n \geq \hat{n}$, $|\mu_n(\bm x)| \leq M$ for some $M > 0$ for all $\bm x \in S_n$. Then, for all $\varepsilon > 0$ there exists $N_\varepsilon$ such that $\sigma_n^2(\bm x) \leq \varepsilon$ for every $\bm x \in S_n$ if $n \geq \hat{n} + N_\varepsilon$. The smallest of such $N_\varepsilon$ is given by
\begin{equation}\label{eq:N_epsilon}
N_\varepsilon = \min\left\{N \in \mathbb{N} : b^{-1}\left(\frac{C\gamma_N}{N}\right) \leq \varepsilon\right\},
\end{equation}
where $b(\eta) \coloneqq\ln(2)\exp\left\{-c_1\frac{M^2}{\eta}\right\}\left[1 - \sqrt{\frac{\sigma_\nu^2}{2c_1\eta + \sigma_\nu^2}}\right]$, $\gamma_N = \max_{D \subset \mathcal{X}; |D| = N}I\left(\bm f(D); \bm y(D)\right)$ is the maximum information capacity of the chosen kernel \citep{srinivas_gaussian_2010,gp_opt_with_mi}, and $C = \ln(2) / \sigma_\nu^2\ln(1 + \sigma_\nu^{-2})$.
\end{restatable}
\begin{proof}
See \cref{appendix_proofs}.
\end{proof}

\cref{lemma:convergence_bound_after_expansion} tells us that if at some point the set of safe parameters $S_n$ stops expanding, then the posterior variance over the safe set vanishes eventually. The intuition behind \cref{lemma:convergence_bound_after_expansion} is that if there were a parameter $\bm x$ in the safe set whose posterior mean remained finite and whose posterior variance remained bounded from below, then an evaluation of $f$ at such $\bm x$ would yield a non negligible average information gain about the safety of $\bm x$, so that, since $\bm x$ is in the safe set, at some point \ourmethod will be forced to choose to evaluate $\bm x$, reducing its posterior variance. This result guarantees that, should the safe set stop expanding, \ourmethod will asymptotically explore the safe set up to an arbitrary $\varepsilon$-accuracy. In practice, we observe that \ourmethod first focuses on reducing the uncertainty in areas of the safe set that are most informative about parameters whose classification is still uncertain (e.g.\@ the boundary of the safe set), and only eventually turns to learning about the inside of the safe set. This behavior is what ultimately leads to the posterior variance to decay over the whole $S_n$. Therefore, even if in general it is not always possible to say whether or not the safe set will ever stop expanding, we can read \cref{lemma:convergence_bound_after_expansion} as an exploration guarantee for \ourmethod, as it rules out the possibility that the proposed acquisition function forever leaves the uncertainty high in areas of the safe set that, if better understood, could lead to an expansion of the safe set.

\cref{lemma:convergence_bound_after_expansion} requires a bound on the GP posterior mean function, which is always satisfied with high probability based on our assumptions about $f$. Specifically, we have that $|\mu_n(\bm x)| \leq 2\beta_n$ with probability of at least $1 - \delta$ for all $n$ (see \cref{appendix_proofs} for details). Therefore, it does not represent an additional restrictive assumption for $f$. Finally, we also note that the the constant $N_\varepsilon$ defined by \cref{eq:N_epsilon} always exists since the function $b$ is monotonically increasing, as long as $\gamma_N$ grows sublinearly in $N$. \citet{srinivas_gaussian_2010} prove that this is the case for commonly-used kernel and, more generally, it is a prerequisite for data-efficient learning with GP models.

\section{Discussion and Limitations}\label{sec:discussion}

\ourmethod drives exploration of the parameter space by selecting the parameters to evaluate according to \cref{eq:x_n_+_1}. An alternative but conceptually similar approach to this criterion would be to consider the parameter that yields the biggest information gain \textit{on average} over the domain, i.e., substituting the inner $\max$ in \cref{eq:x_n_+_1} with an average over $\mathcal{X}$. The resulting integral, however, is intractable and would require further approximations. Moreover, the parameter found by solving \cref{eq:x_n_+_1} will also yield a high average information gain over the domain, due to the regularity of all involved objects.

Being able to work in a continuous domain, \ourmethod can deal with higher dimensional domains better than algorithms requiring a discrete parameter space. However, as noted in \cref{sec:theory}, finding $\bm x_{n+1}$ as in \cref{eq:x_n_+_1} means to solve a non-convex optimization problem with twice the dimension of the the parameter space, which can also become a computationally challenging problem as the dimension grows. In a high-dimensional setting, we follow \linebo by \citet{kirschner_adaptive_2019}, which at each iteration selects a random one-dimensional subspace to which it restricts the optimization of the acquisition function.

In \cref{algorithm,sec:problem_statement}, we assumed the observation process to be homoskedastic. However, it needs not to be the case, and the results can be extended to the case of heteroskedastic Gaussian noise. The observation noise at a parameter $\bm x$ explicitly appears in the \ourmethod acquisition function, since it crucially affects the amount of information that we can gain by evaluating the constraint $f$ at $\bm x$. On the contrary, \stageopt-like methods do not consider the observation noise in their acquisition functions. As a consequence, \ourmethod can perform significantly better in an heteroskedastic setting, as we also show in \cref{experiments}.

Lastly, we reiterate that the theoretical safety guarantees offered by \ourmethod are derived under the assumption that $f$ is a bounded norm element of the RKHS space associated with the GP's kernel. In applications, therefore, the choice of the kernel function becomes even more crucial when safety is an issue. For details on how to construct and choose kernels see \citep{garnett_bo_book}. The safety guarantees also depend on the choice of $\beta_n$. Typical expressions for $\beta_n$ include the RKHS norm of the constraint $f$ \citep{kernelized_bandits,fiedler_beta_bounds}, which is in general difficult to estimate in practice. Because of this, usually in practice a constant value of $\beta_n$ is used instead.
 
\section{Experiments}\label{experiments}

In this section we empirically evaluate \ourmethod. Additional details about the experiments and setup can be found in \cref{appendix_experiments}. As commonly done in the literature (see \cref{sec:discussion}), we set $\beta_n = 2$ for all experiments. This choice guarantees safety per iteration, rather than jointly for all $n$ and it allows for a less conservative bound than the one needed for the joint guarantees.

\paragraph{GP samples}
For the first part of the experiments, we evaluate \ourmethod on constraint functions $f$ that we obtain by sampling a GP prior at a finite number of points. This allows us to test \ourmethod under the assumptions of the theory and we compare its performance to that of the exploration part of \stageopt \citep{sui2018stagewise}. \stageopt is a modified version of \safeopt, in which the exploration and optimization parts are performed separately: first the \safeopt exploration strategy is used to expand the safe set as much as possible, then the objective function is optimized within the discovered safe set. We further modify the version of \stageopt that we use in the experiment by defining the safe set in the same way \ourmethod does, i.e., by means of the GP posterior, as done, for example, also by \citet{safe_quadrotors_felix}. We select 100 samples from a two-dimensional GP with RBF kernel, defined in $[-2.5, 2.5]\times[-2.5, 2.5]$ and run \ourmethod and \stageopt for 100 iterations for each sample. As \stageopt requires a discretization of the domain, we use this discretization to compare the sample efficiency of the two methods, by computing, at each iteration, what percentage of the discretized domain is classified as safe. Moreover, we also compare with the heuristic acquisition inspired by \safeopt proposed by \citet{safe_quadrotors_felix}. This method works exactly as \stageopt, with the difference that the set of expanders is computed using directly the GP posterior, rather than the Lipschitz constant. More precisely, a parameter $\bm x$ is considered an expander if observing a value of $\mu_n(\bm x) + \beta_n\sigma_n(\bm x)$ at $\bm x$ would enlarge the safe set. For the \stageopt run, we use the kernel metric to compute the set of potential expanders, for different values of the Lipschitz constant $L$. From the results shown in \cref{fig:comparison}, we see not only that \ourmethod performs as well or better than all tested instances of \stageopt, but also how the choice of $L$ affects the performance of the latter. This plot makes it also evident how crucial the choice of the Lipschitz constant is for \stageopt and \safeopt-like algorithms in general. 
In \cref{tab:safety_violations}, in \cref{appendix_experiments}, we report the average percentage of safety violations per run achieved by \ourmethod and \stageopt. As expected, we see that the percentage of safety violations is comparable among all algorithms.

\begin{figure*}
\hfill
     \centering
     \begin{subfigure}{0.495\textwidth}
         \centering
         \includegraphics[width=\textwidth,height=0.15\textheight]{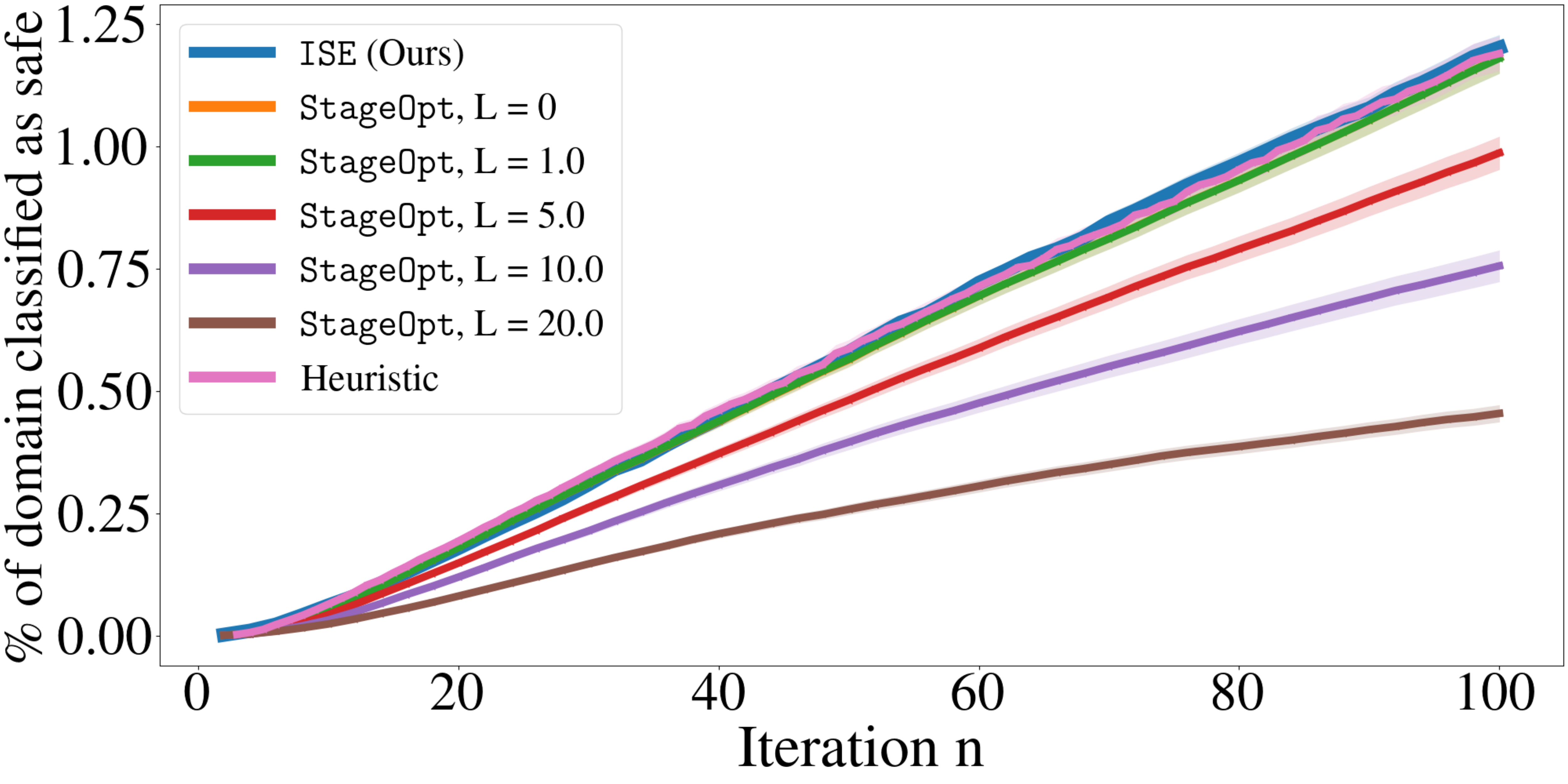}
         \caption{Comparison with \stageopt for GP samples.}
         \label{fig:comparison}
     \end{subfigure}
     \begin{subfigure}{0.495\textwidth}
         \centering
         \includegraphics[width=\textwidth,height=0.15\textheight]{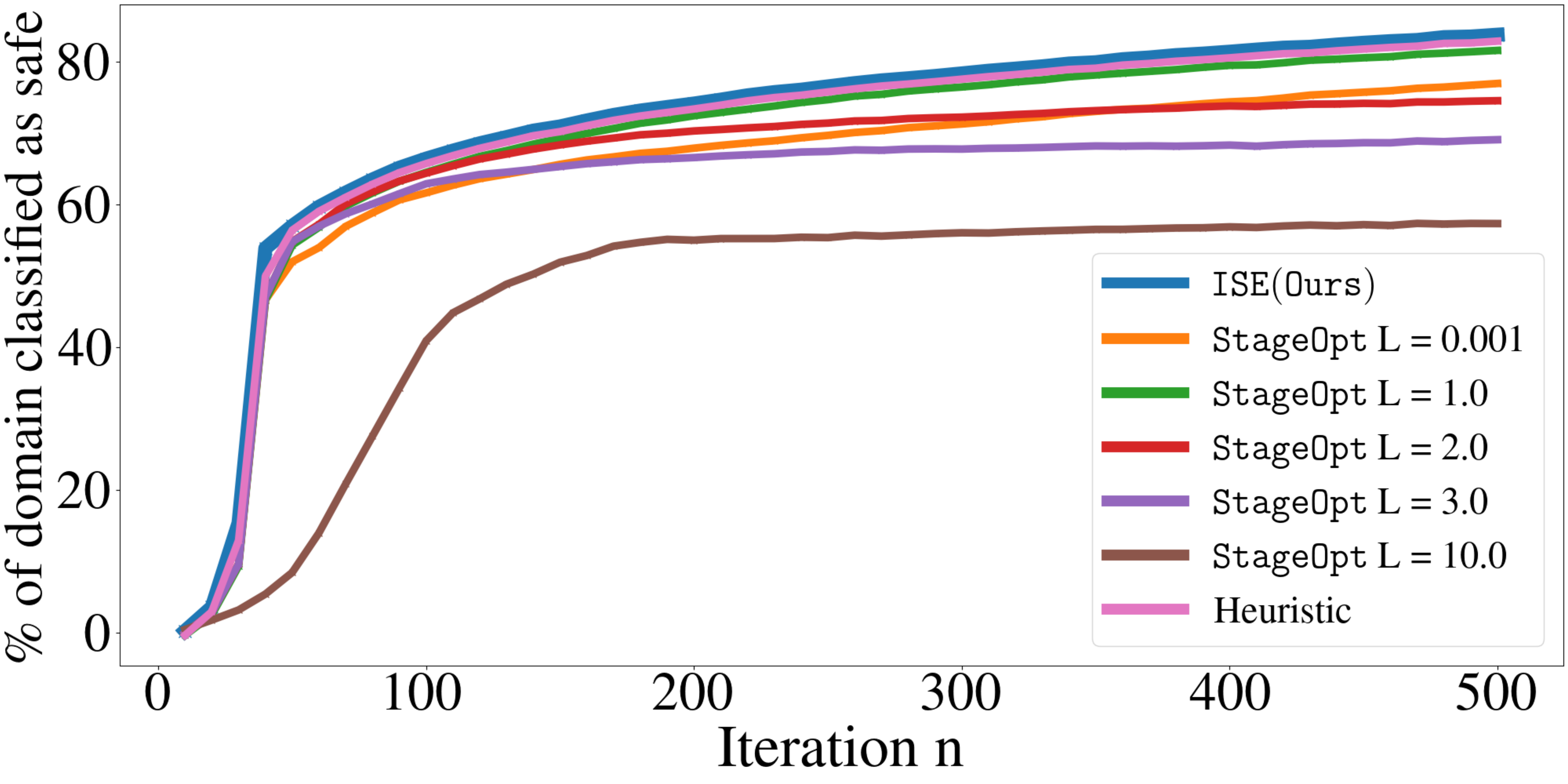}
         \caption{Comparison in 1D exponential example.}
         \label{fig:one_d_exploration}
     \end{subfigure}
        \caption{(\subref{fig:comparison}) Average expansion of safe set over 100 two-dimensional GP samples. The average percentage of the domain classified as safe is plotted as a function of $n$ with its standard error. The lines for $L = 0$ and $L = 1$ overlap. We can see that \ourmethod obtains an higher sample efficiency than the best instance of \stageopt and a comparable one with the heuristic acquisition function proposed by \citet{safe_quadrotors_felix}. The plot also shows that \stageopt performance is heavily affected by the choice of $L$. In (\subref{fig:one_d_exploration}) the average percentage of the domain classified as safe in the one dimensional example for $f(x) = e^{-x} + 0.05$ is plotted as function of th eiteration $n$ with its standard error, and it shows the detrimental effect of over- and underestimating $L$.}
        \label{fig:safe_set_expansion_comparison}
\end{figure*}

To show that for \stageopt exploration not only overestimating the Lipschitz constant, but also underestimating it can negatively impact performance, we consider the simple one-dimensional constraint function $f(x) = e^{-x} + 0.05$ and run the safe exploration for multiple values of the Lipschitz constant. This function gets increasingly away from the safety threshold for $x \to -\infty$, while it asymptotically approaches the threshold for $x \to \infty$, so that a good exploration algorithm would, ideally, quickly classify as safe the domain region for $x < 0$ and then keep exploring the boundary of the safe set for $x > 0$. The results plotted in \cref{fig:one_d_exploration} show how both a too high and a too low Lipschitz constant can lead to sub-optimal exploration. In the case of a too small constant, this is because \stageopt considers expanders almost all parameters in the domain, leading to additional evaluations in the region for $x < 0$ that are unlikely to cause expansion of the safe set. On the other hand, a too high value of the Lipschitz constant can lead to the set of expanders to be empty as soon as the posterior mean gets close to the safety threshold for $x > 0$.

\paragraph{OpenAI Gym Control}
After investigating the performance of \ourmethod under the hypothesis of the theory, we apply it to two classic control tasks from the OpenAI Gym framework \citep{brockman2016openai}, with the goal of finding the set of parameters of a controller that satisfy some safety constraint. In particular we consider linear controllers for the inverted pendulum and cart pole tasks.

For the inverted pendulum task, the linear controller is given by $u_t = \alpha_1 \theta_t + \alpha_2 \dot{\theta}_t$, where $u_t$ is the control signal at time $t$, while $\theta_t$ and $\dot{\theta}_t$ are, respectively, the angular position and the angular velocity of the pendulum. Starting from a position close to the upright equilibrium, the controller's task is the stabilization of the pendulum, subject to a safety constraint on the maximum velocity reached within one episode. For some given initial controller configuration $\bm \alpha_0 \coloneqq (\alpha_1^0, \alpha_2^0)$, we want to explore the controller's parameter space avoiding configurations that lead the pendulum to swing with a too high velocity. We apply \ourmethod to explore the $\bm \alpha$-space with $\bm x_0 = \bm \alpha_0$ and the safety constraint being the maximum angular velocity reached by the pendulum in an episode of fixed length. In this case, the safety threshold is not at zero, but rather at some finite value $\dot{\theta}_M$, and the safe parameters are those for which the maximum velocity is below $\dot{\theta}_M$. The formalism developed in the previous sections can be easily applied to this scenario if we consider $f(\bm \alpha) = - (\max_t \dot{\theta}_t(\bm \alpha) - \dot{\theta}_M)$. In \cref{fig:true_pendulum_safe_set} we show the true safe set for this problem, while in \cref{fig:pendulum_discover_safe_set_1,fig:pendulum_discover_safe_set_2,fig:pendulum_discover_safe_set_3} one can see how \ourmethod safely explores the true safe set. These plots show how the \ourmethod acquisition function \cref{eq:x_n_+_1} selects parameters that are close to the current safe set boundary and, hence, most informative about the safety of parameters outside of the safe set. This behavior eventually leads to the full true safe set to be classified as safe by the GP posterior, as \cref{fig:pendulum_discover_safe_set_3} shows.

The cart pole task is similar to the inverted pendulum one, but the parameter space has three dimensions. The controller we consider is given by $u_t = \alpha_1 \theta_t + \alpha_2 \dot{\theta_t} + \alpha_3 \dot{s_t}$, where $\theta_t$ and $\dot{\theta}_t$ are, respectively, the angular position and angular velocity of the pole at time $t$, while $\dot{s}_t$ is the cart's velocity. We set the initial state to zero angular and linear velocity and with the pole close to the vertical position, with the controller's goal being to keep the pole stable in the upright position. A combination of the three parameters $\alpha_1$, $\alpha_2$ and $\alpha_3$ is considered safe if the angle of the pole does not exceed a given threshold within the episode. Again, we can easily cast this safety constraint in terms of the formalism developed in the paper: $f(\bm \alpha) = - (\max_t \theta_t(\bm \alpha) - \theta_M)$, where $\theta_M$ is the maximum allowed angle. \cref{fig:cartpole_safe_set} shows the expansion of the cart pole $\bm \alpha$ space promoted by \ourmethod, compared with \stageopt for different values of the Lipschitz constant. Both methods achieve a comparable sample efficiency and both lead to the classification as safe of the full true safe set.

\begin{figure}[ht]
\hfill
     \centering
     \begin{subfigure}[b]{0.49\textwidth}
         \centering
         \includegraphics[width=\textwidth,height=0.165\textheight]{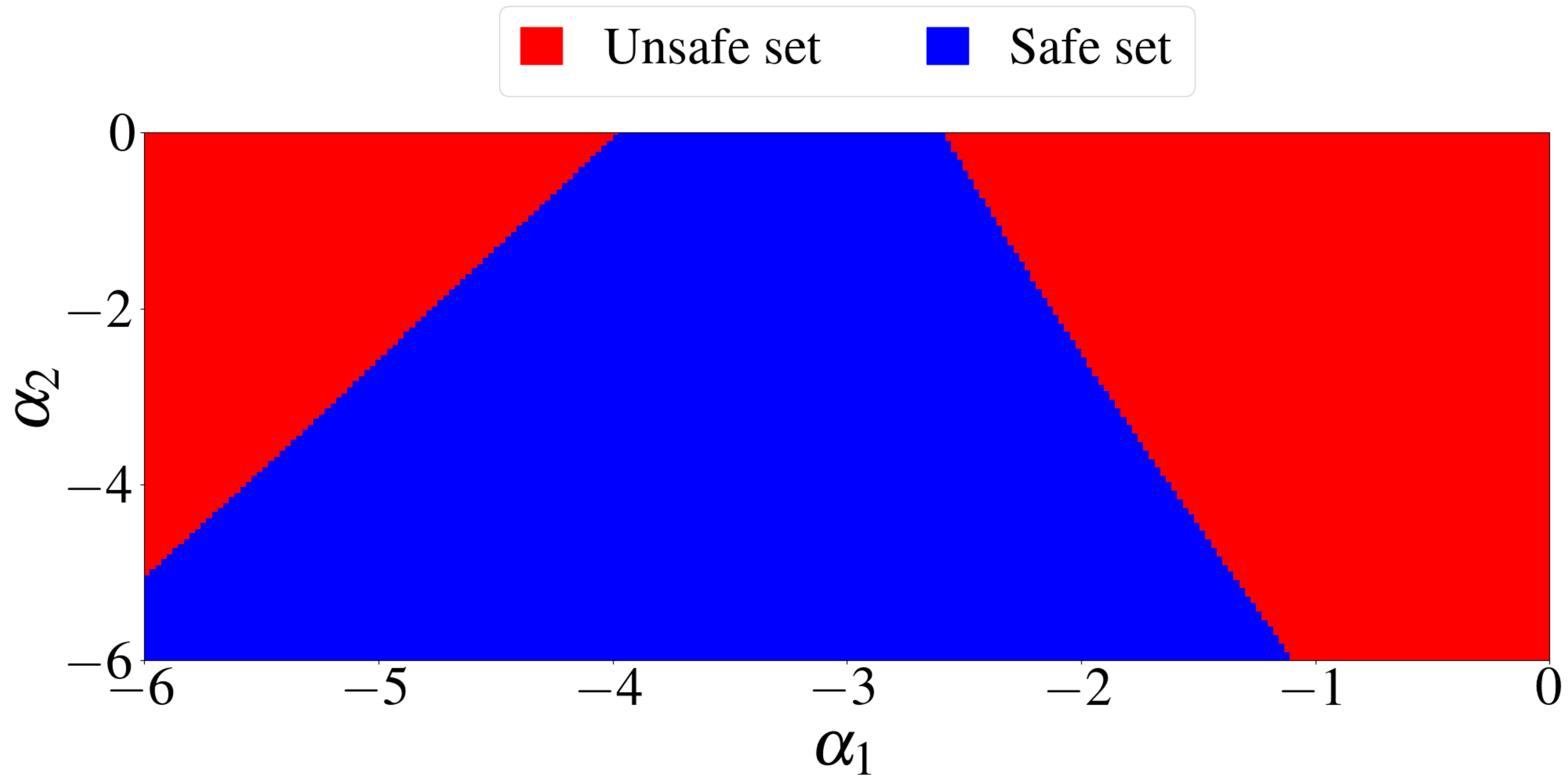}
         \caption{True safe set}
         \label{fig:true_pendulum_safe_set}
     \end{subfigure}
     \hfill
     \begin{subfigure}[b]{0.49\textwidth}
         \centering
         \includegraphics[width=\textwidth,height=0.165\textheight]{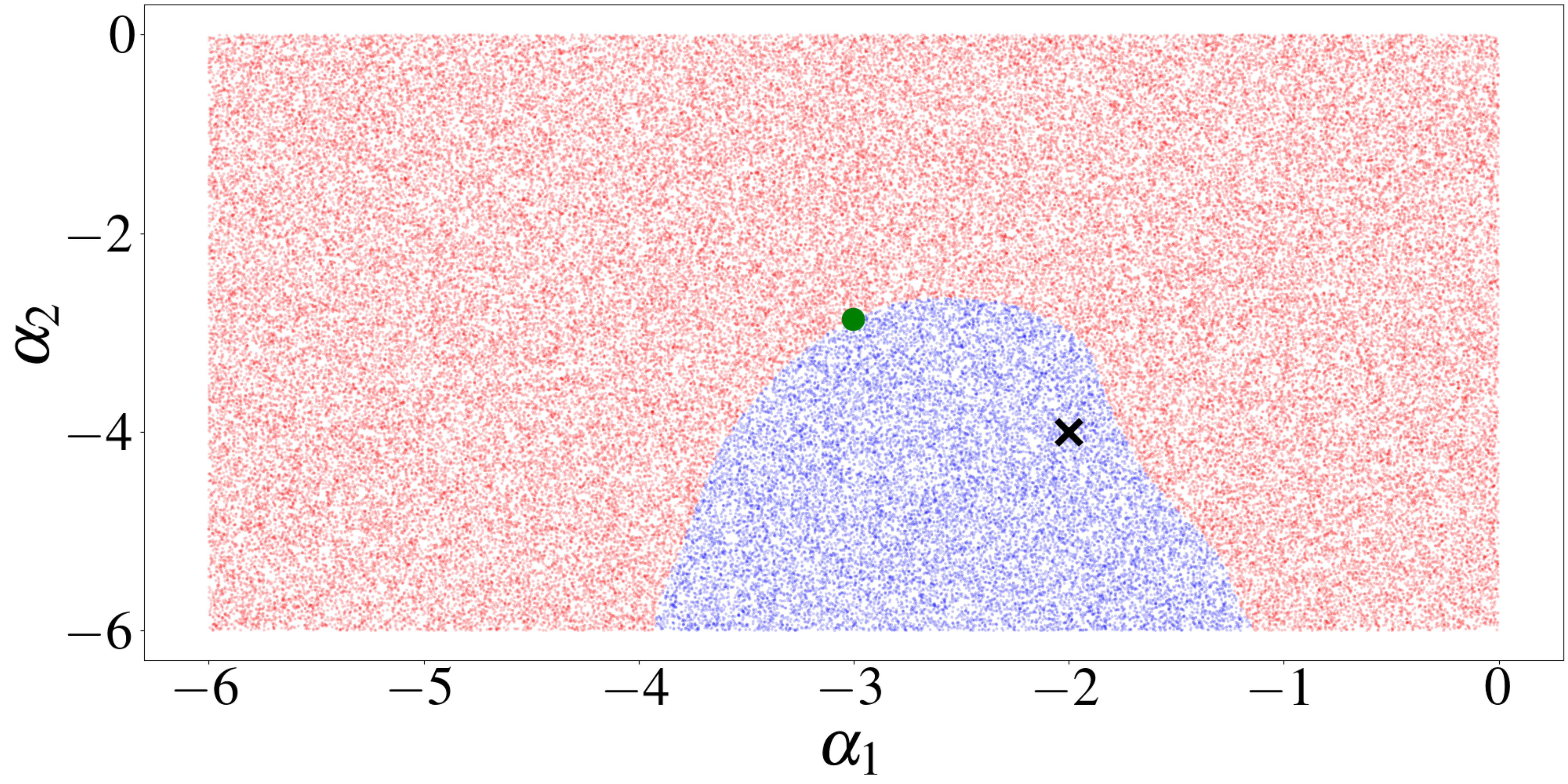}
         \caption{Iteration 10}
         \label{fig:pendulum_discover_safe_set_1}
     \end{subfigure}
     \hfill
     \begin{subfigure}[b]{0.49\textwidth}
         \centering
         \includegraphics[width=\textwidth,height=0.165\textheight]{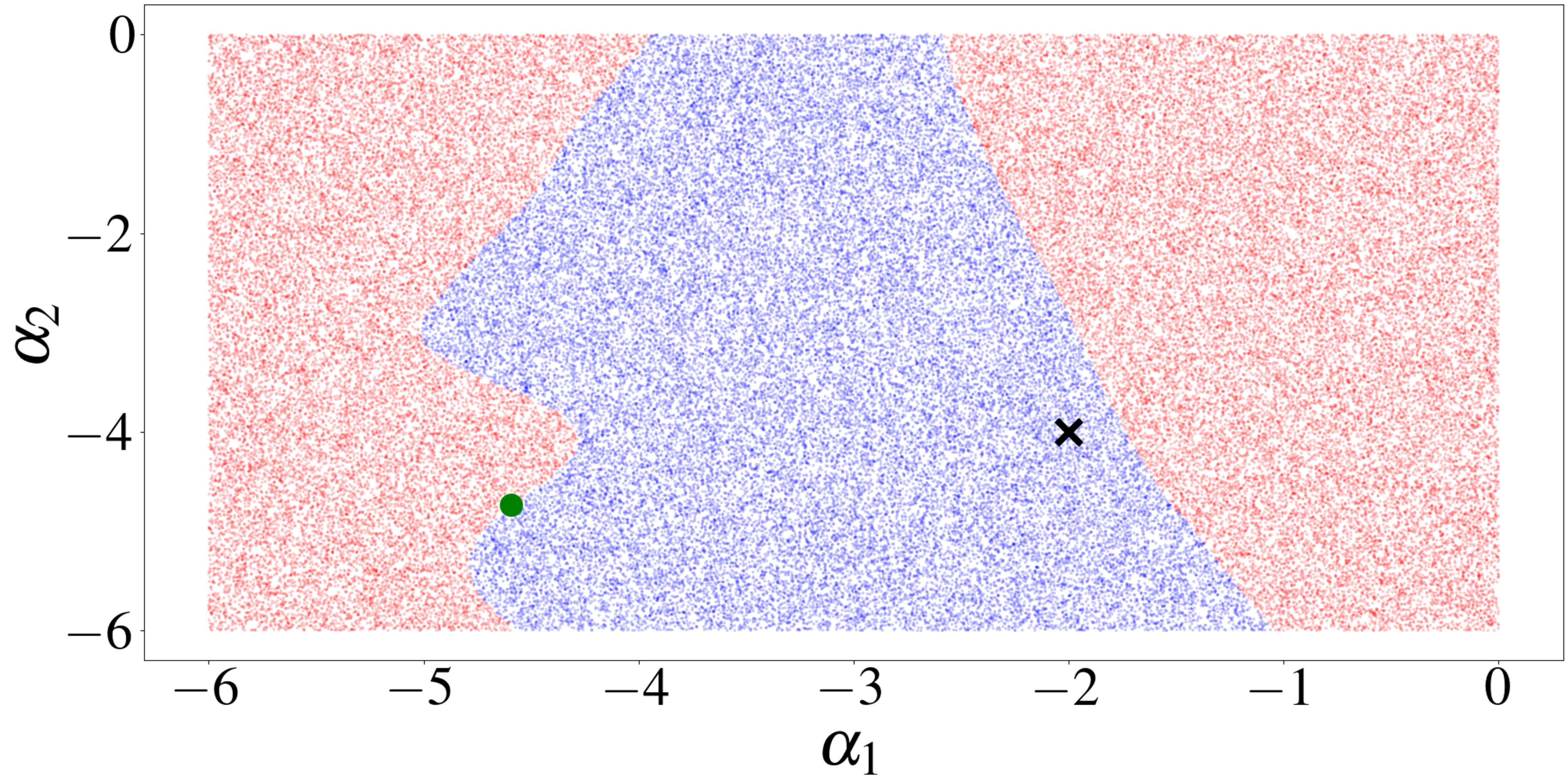}
         \caption{Iteration 30}
         \label{fig:pendulum_discover_safe_set_2}
     \end{subfigure}
     \hfill
     \begin{subfigure}[b]{0.49\textwidth}
         \centering
         \includegraphics[width=\textwidth,height=0.165\textheight]{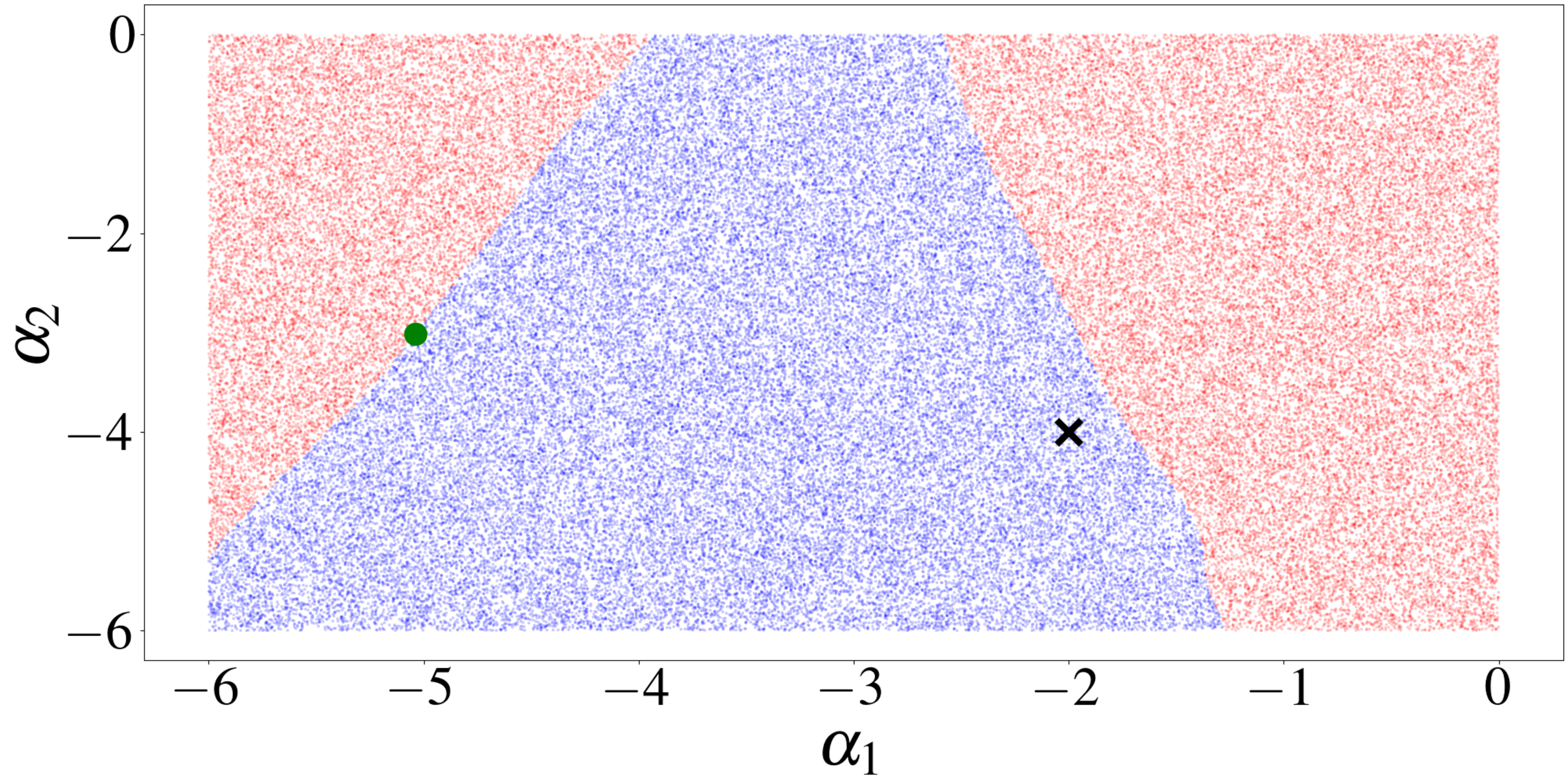}
         \caption{Iteration 50}
         \label{fig:pendulum_discover_safe_set_3}
     \end{subfigure}
        \caption{Safe exploration of the linear controller's parameter space in the inverted pendulum experiment. In (\subref{fig:true_pendulum_safe_set}) we see the true safe set, while in (\subref{fig:pendulum_discover_safe_set_1}-\subref{fig:pendulum_discover_safe_set_3}) we see the safe set (blue region) as identified by \ourmethod at various iterations. The point marked by the green dot is $\bm \alpha_{n+1}$ as selected by \cref{eq:x_n_+_1}, while the black cross is the initial safe seed $\bm \alpha_0$.}
        \label{fig:pendulum_experiment}
\end{figure}

\begin{figure*}
\hfill
     \centering
     \begin{subfigure}{0.495\textwidth}
         \centering
         \includegraphics[width=\textwidth,height=0.16\textheight]{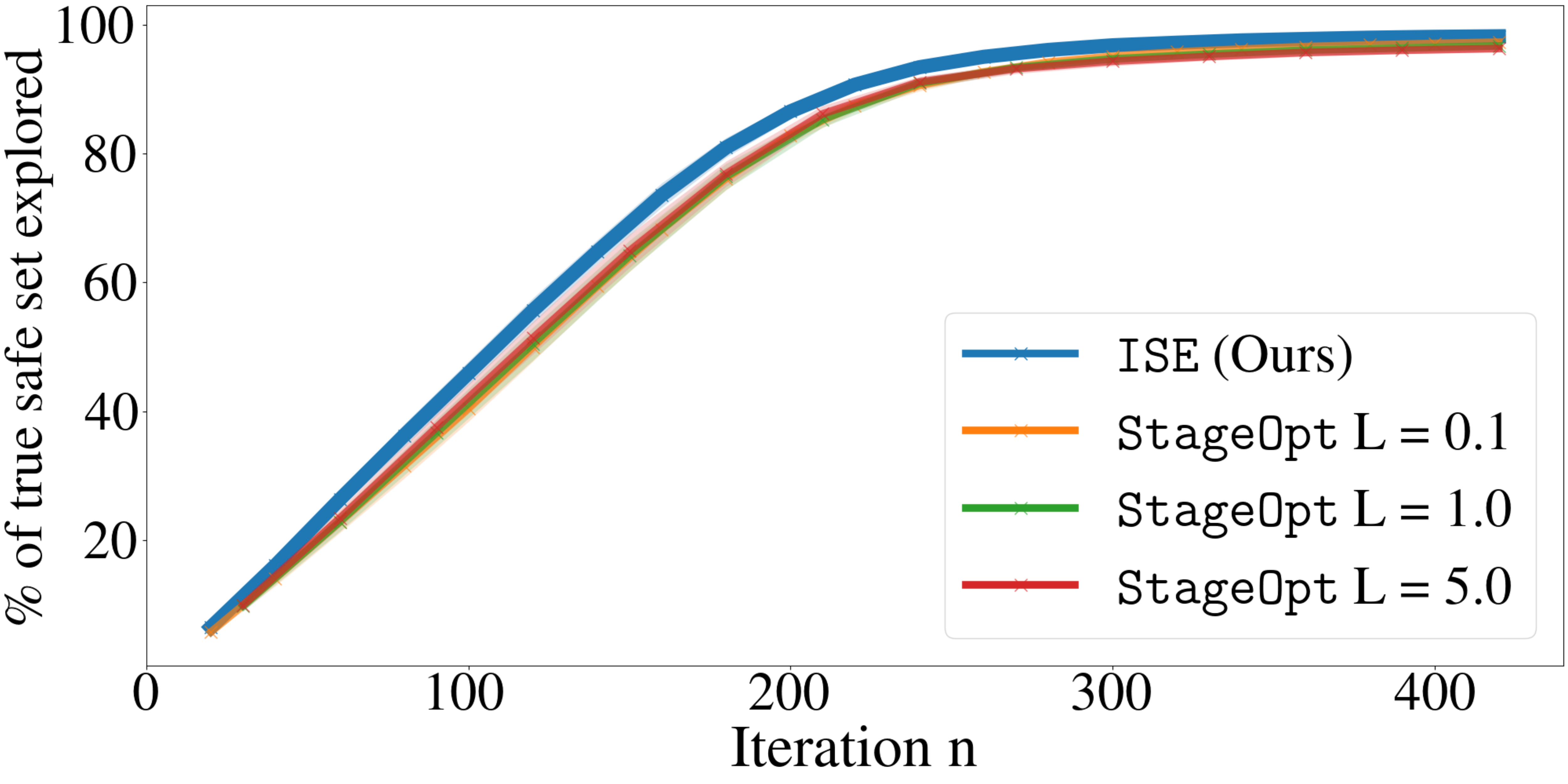}
         \caption{Exploration of the safe set for the cart pole task.}
         \label{fig:cartpole_safe_set}
     \end{subfigure}
     \begin{subfigure}{0.495\textwidth}
         \centering
         \includegraphics[width=\textwidth,height=0.16\textheight]{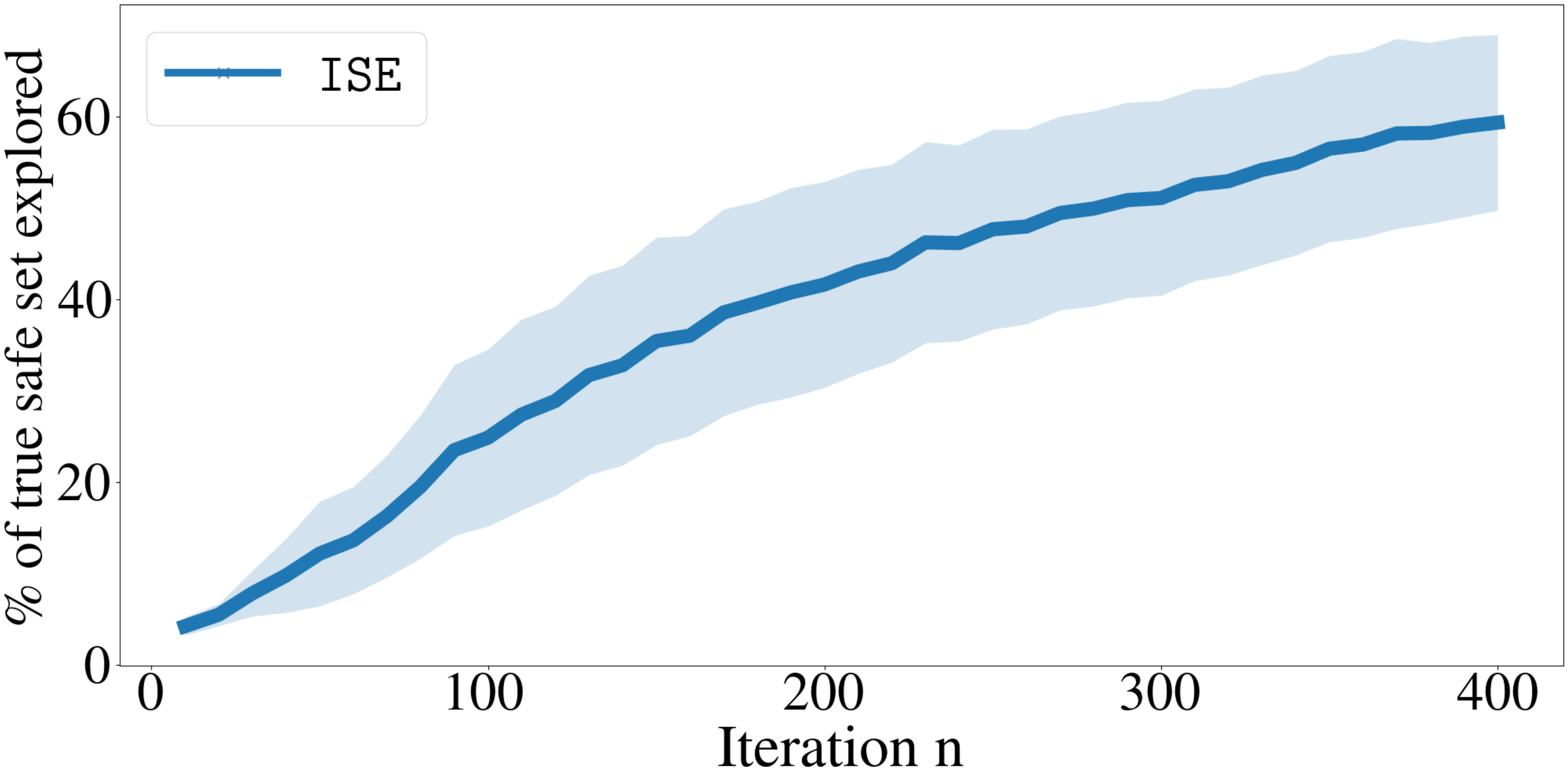}
         \caption{Expansion of the safe set in dimension five.}
         \label{fig:five_d_expansion}
     \end{subfigure}
        \caption{In (\subref{fig:cartpole_safe_set}) we plot the percentage of the true safe set of the cart pole task classified as safe by \ourmethod and \stageopt, while in (\subref{fig:five_d_expansion}) we see the expansion of the five-dimensional safe set promoted by \ourmethod, for the safety constraint $f$ used in the high dimensional experiments.}
        \label{fig:cartpole_experiment_and_5_d}
        \hfill
\end{figure*}

\paragraph{High dimensional domains}
Many interesting applications have a high dimensional parameter space. While \safeopt -like methods are difficult to apply already with dimension $>3$ due to the discretization of the domain, \ourmethod can perform well also in four or five dimensions. To see this, we apply \ourmethod to the constraint function $f(\bm x) = e^{-\bm x^2} + 2e^{-(\bm x - \bm x_1)^2} + 5e^{-(\bm x - \bm x_2)^2} - 0.2$. \cref{fig:five_d_expansion} shows the \ourmethod performance in dimension 5. We see that \ourmethod is able to promote the expansion of the safe set, leading to an increasingly bigger portion of the true safe set to be classified as safe. 

\paragraph{Heteroskedastic noise domains}\label{par:high_heteroskedastic}
For even higher dimensions, we can follow a similar approach to \linebo, limiting the optimization of the acquisition function to a randomly selected one-dimensional subspace of the domain. Moreover, as discussed in \cref{sec:discussion}, it is also interesting to test \ourmethod in the case of heteroskedastic observation noise, since the noise is a critical quantity for the \ourmethod acquisition function, while it does not affect the selection criterion of \stageopt-like methods. Therefore, in this experiment we combine a high dimensional problem with heteroskedastic noise. In particular, we apply a \linebo version of \ourmethod to the constraint function $f(\bm x) = \frac{1}{2}e^{-\bm x^2} + e^{-(\bm x \pm \bm x_1)^2} + 3e^{-(\bm x \pm \bm x_2)^2} + 0.2$ in dimension nine and ten, with the safe seed being the origin. This function has two symmetric global optima at $\pm \bm x_2$ and we set two different noise levels in the two symmetric domain halves containing the optima. To assess the exploration performance, we use the simple regret, defined as the difference between the current safe optimum and the true safe optimum. As the results in \cref{fig:high_d_line} show, \ourmethod achieve a greater sample efficiency than the other \stageopt -like methods. Namely, for a given number of iterations, by explicitly exploiting knowledge about the observation noise, \ourmethod is able to classify as safe regions of the domain further away from the origin, in which the constraint function assumes its largest values, resulting in a smaller regret. On the other hand, \safeopt-like methods only focus on the posterior variance, so that the higher observation noise causes them to remain stuck in a smaller neighborhood of the origin, resulting in bigger regret.

\begin{figure*}
\hfill
     \centering
     \begin{subfigure}{0.495\textwidth}
         \centering
         \includegraphics[width=\textwidth,height=0.16\textheight]{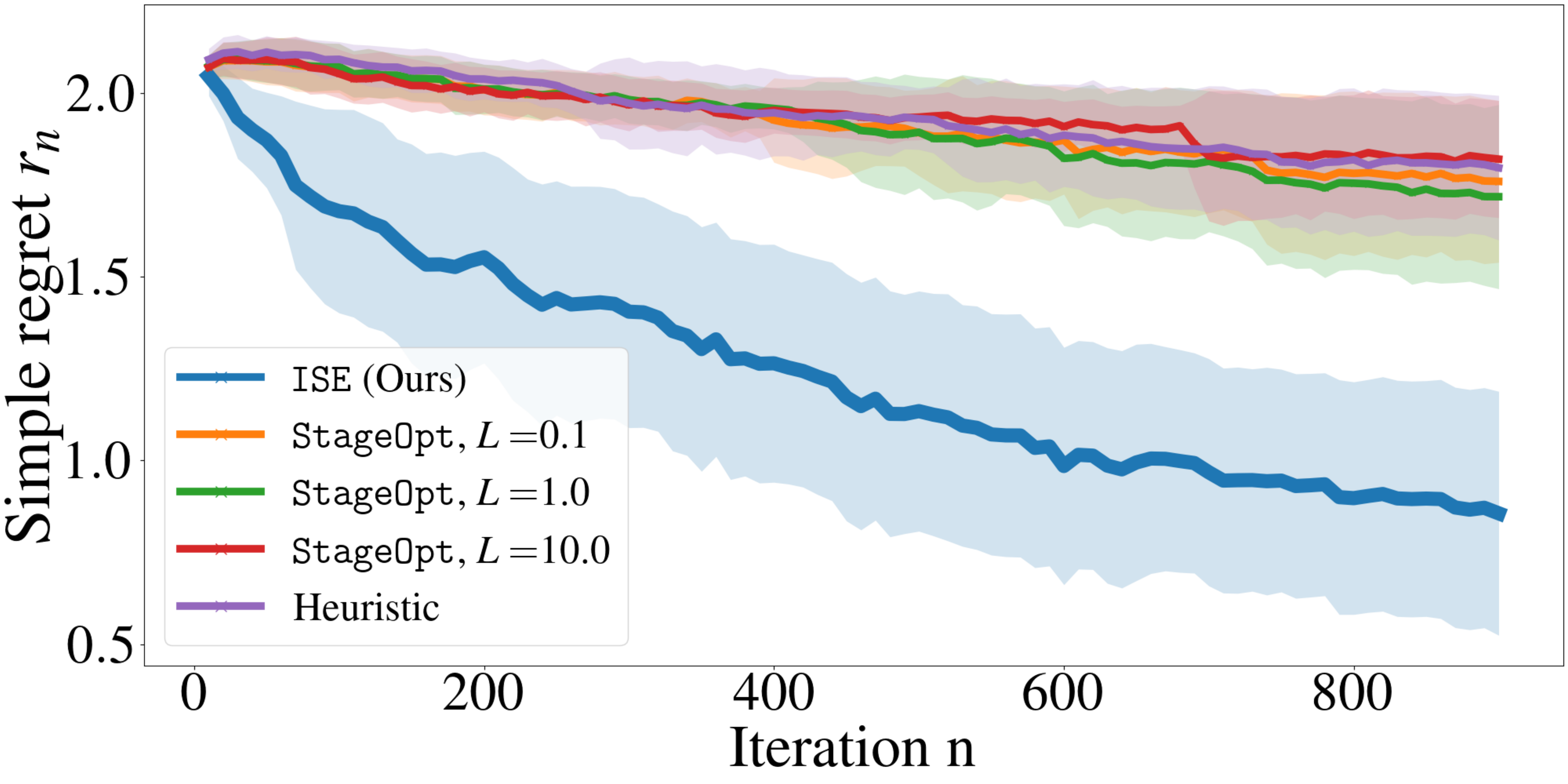}
         \caption{Dimension 9.}
         \label{fig:high_d_line_9}
     \end{subfigure}
     \begin{subfigure}{0.495\textwidth}
         \centering
         \includegraphics[width=\textwidth,height=0.16\textheight]{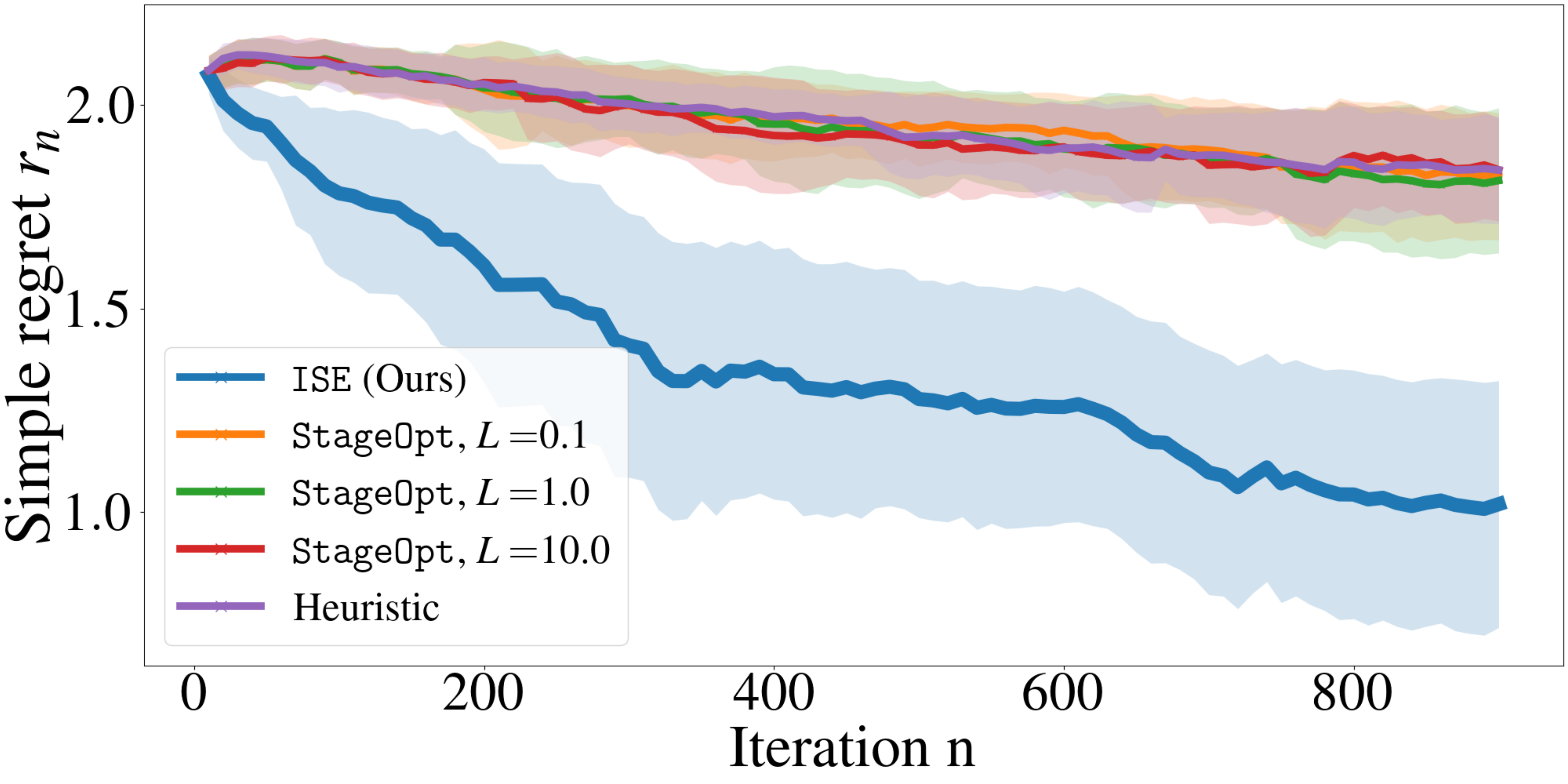}
         \caption{Dimension 10.}
         \label{fig:high_d_line_10}
     \end{subfigure}
        \caption{Example of high dimensional exploration, for $d=$9 and $d=$10. After every ten iterations, we perform safe optimization with UCB acquisition within the current safe set and plot the simple regret $r_n$ with respect to the safe optimum. (\subref{fig:high_d_line_9}) and (\subref{fig:high_d_line_9}) show, respectively, the average regret over 70 runs in dimension nine and ten, as a function of the number of iterations. We can see that this adapted version of \ourmethod promotes expansion of the safe set, leading to classifying as safe regions where the latent function attains its largest value. The plots also show that \ourmethod achieves a better sample efficiency than both \stageopt-like exploration and the \stageopt inspired heuristic acquisition.}
        \label{fig:high_d_line}
\end{figure*}

\section{Conclusion and Societal Impact}\label{sec:conclusion}

We have introduced \ourmethodlong (\ourmethod), a novel approach to safely explore a space in a sequential decision task where the safety constraint is \textit{a priori} unknown. \ourmethod efficiently and safely explores by evaluating only parameters that are safe with high probability and by choosing those parameters that yield the greatest information gain about the safety of other parameters. We theoretically analyzed \ourmethod and showed that it leads to arbitrary reduction of the uncertainty in the largest reachable safe set containing the starting parameter. Our experiments support these theoretical results and demonstrate an increased sample efficiency and scalability of \ourmethod compared to \safeopt-based approaches.

In many safety sensitive applications the shape of the safety constraints is unknown, so that an important prerequisite for any kind of process is to identify what parameters are safe to evaluate. By providing a principled way to do this, the contributions of this paper allow to deal with safety in a broad range of applications, which can favor the usage of ML approaches also in safety sensitive settings. On the other hand, misuse of the proposed method cannot be prevented in general.

\bibliography{references_final}
\bibliographystyle{plainnat}

\newpage

\appendix

\section{Proofs}\label{appendix_proofs}

This appendix contains the proofs of the results found in \cref{sec:theory}. We start by introducing a useful rewriting of the mutual information $\hat{I}_n(\{\bm x, y\}; \Psi(\bm z))$ as given by \cref{eq:approximated_psi_entropy,eq:averaged_post_measurement_entropy}. We then use this expression to prove some results about $\hat{I}_n(\{\bm x, y\}; \Psi(\bm z))$ needed for the proof of \cref{lemma:convergence_bound_after_expansion}. Finally, we formally show the safety guarantee offered by \ourmethod and the claim that the posterior mean is bounded by $2\beta_n$ with high probability.

\begin{lemma}\label{lemma:rewriting_of_mutual_info}
The mutual information $\hat{I}_n(\{\bm x, y\}; \Psi(\bm z))$ as given by \cref{eq:approximated_psi_entropy,eq:averaged_post_measurement_entropy} can be rewritten as
\begin{equation}\label{eq:rewritten_mutual_info}
\ln(2)\left[\exp\left\{-c_1\frac{\mu_n^2(\bm z)}{\sigma_n^2(\bm z)}\right\} - \sqrt{\frac{1 - \rho_\nu^2(\bm x)\rho_n^2(\bm x, \bm z)}{1 + c_2\rho_\nu^2(\bm x)\rho_n^2(\bm x, \bm z)}}\exp\left\{-c_1\frac{\mu_n^2(\bm z)}{\sigma_n^2(\bm z)}\frac{1}{1 + c_2\rho_\nu^2(\bm x)\rho_n^2(\bm x, \bm z)}\right\}\right]
\end{equation}
where $c_1 = 1 / \ln(2)\pi$ and $c_2 = 2c_1 -1$, and where $\rho_\nu^2(\bm x)$ is given by
\begin{equation}\label{eq:rho_nu_def}
\rho_\nu^2(\bm x) \coloneqq \frac{\sigma_n^2(\bm x)}{\sigma_\nu^2 + \sigma_n^2(\bm x)}
\end{equation}
where the dependency on $n$ has been dropped in the notation for simplicity.
\end{lemma}

\begin{proof}
It suffices to substitute the expression \cref{eq:rho_nu_def} for $\rho_\nu^2(\bm x)$ in the second term of \cref{eq:rewritten_mutual_info} to recover \cref{eq:averaged_post_measurement_entropy}. The claim follows then directly from \cref{eq:approximated_psi_entropy} and the definition of mutual information.
\end{proof}

\begin{lemma}\label{lemma:monotonicity_ratio}
The mutual information $\hat{I}_n(\{\bm x, y\}; \Psi(\bm z))$ as given by \cref{eq:rewritten_mutual_info} is monotonically decreasing in $\mu_n^2(\bm z) / \sigma_n^2(\bm z)$ $\forall \bm x, \bm z \in \mathcal{X}$.
\end{lemma}

\begin{proof}
First of all, let us simplify notation and define $R^2 \coloneqq \mu^2_n(\bm z) / \sigma^2_n(\bm z)$ and $\tilde{\rho}^2 \coloneqq \rho_\nu^2(\bm x)\rho^2_n(\bm x, \bm z)$. We then need to show that:
\begin{equation}\label{eq:partial_ratio}
\frac{\partial}{\partial R^2}\left[\exp\left\{-c_1R^2\right\} - \sqrt{\frac{1 - \tilde{\rho}^2}{1 + c_2\tilde{\rho^2}}}\exp\left\{-c_1R^2\frac{1}{1 + c_2\tilde{\rho}^2}\right\}\right] < 0 ~\forall R, ~\forall \tilde{\rho} \in [0, 1]
\end{equation}
We then can compute the derivative and ask under which conditions it is non negative. Requiring \cref{eq:partial_ratio} to be non negative is equivalent to ask that:
\begin{equation}
R^2 \leq \frac{1 + c_2\tilde{\rho}^2}{c_1c_2\tilde{\rho}^2}\left[\ln\left(1 + c_2\tilde{\rho}^2\right) + \frac{1}{2}\ln\left(\frac{1 + c_2\tilde{\rho}^2}{1 - \tilde{\rho}^2}\right)\right]
\end{equation}
Now, we observe that, since $c_2 \in (-1, 0)$, while $c_1>0$ and $\tilde{\rho}^2 \in [0, 1]$, the factor $(1 + c_2\tilde{\rho}^2) / c_1c_2\tilde{\rho}^2$ is always negative. For what concerns the sum of logarithms in the square brackets, it is strictly positive $\forall \tilde{\rho}^2 \in [0, 1]$, which means that, for \cref{eq:partial_ratio} to be non negative, we would need $R^2 < 0$, which is impossible, given that $R \in \mathbb{R}$.
\end{proof}

\begin{lemma}\label{lemma:monotonicity_rho}
The mutual information $\hat{I}_n(\{\bm x, y\}; \Psi(\bm z))$ as given by \cref{eq:approximated_psi_entropy} and \cref{eq:averaged_post_measurement_entropy} is monotonically increasing in $\rho_\nu^2(\bm x)\rho_n^2(\bm x, \bm z)$ $\forall \bm x, \bm z \in \mathcal{X}$.
\end{lemma}

\begin{proof}
As in the proof of \cref{lemma:monotonicity_ratio}, let us define $R^2 \coloneqq \mu^2_n(\bm z) / \sigma^2_n(\bm z)$ and $\tilde{\rho}^2 \coloneqq \rho_\nu^2(\bm x)\rho^2_n(\bm x, \bm z)$.
Analogously to the proof of \cref{lemma:monotonicity_ratio}, we compute the partial derivative of $\hat{I}_n\left(\{\bm x, y\}; \Psi(\bm z)\right)$ with respect to $\tilde{\rho}^2$ and show that it is strictly positive $\forall \tilde{\rho} \in [0, 1]$ and $\forall R$. The partial derivative is given by:
\begin{equation}\label{eq:partial_rho}
\frac{\partial}{\partial \tilde{\rho}^2}\hat{I}_n\left(\{\bm x, y\}; \Psi(\bm z)\right) = -\frac{1}{2}\frac{\exp\left\{\frac{c_1R^2}{|c_2|\tilde{\rho}^2 - 1}\right\}\left[|c_2|^2\tilde{\rho}^2 + |c_2|\left(-2c_1R^2(\tilde{\rho}^2 - 1) - \tilde{\rho}^2 - 1 \right) + 1\right]}{\sqrt{\frac{\tilde{\rho}^2 - 1}{|c_2|\tilde{\rho}^2 - 1}} \left(|c_2|\tilde{\rho}^2 - 1\right)^3}
\end{equation}
we now have to ask when this derivative is non positive. After remembering that $|c_2| < 1$ and that $\tilde{\rho}^2 \in [0, 1]$, we see that the denominator is always negative; we also have that the exponential term in the numerator is always positive. These two facts, together with the minus sign, imply that, for \cref{eq:partial_rho} to be $\leq 0$, we need the term inside the square brackets to be non positive. This requirement leads to the condition:
\begin{equation}\label{eq:condition_on_rho}
\tilde{\rho}^2 \geq \frac{c_1\pi^2\ln^2(2)R^2 - 2c_1\pi\ln(2)R^2 + \pi\ln(2)}{\left(\pi\ln(2) - 2\right) \left(c_1\pi\ln(2)R^2 + 1\right)}
\end{equation}
where we have used the explicit value of $c_2$: $2c_1 - 1$. Finally, the rhs of \cref{eq:condition_on_rho} is always above 1 for $R^2 \geq 0$, which concludes the proof, since $\tilde{\rho}^2 \in [0, 1]$.
\end{proof}

\begin{lemma}\label{lemma:rho_nu_increasing_in_sigma}
For every value of $\sigma_\nu^2 > 0$, $\rho_\nu^2(\bm x)$ as defined in \cref{eq:rho_nu_def} is monotonically increasing in $\sigma_n(\bm x)$. 
\end{lemma}

\begin{proof}
As for \cref{lemma:monotonicity_ratio,lemma:monotonicity_rho}, we compute the derivative of $\rho_\nu^2(\bm x)$ with respect to $\sigma_n^2(\bm x)$ and show that it is strictly positive $\forall \sigma_n^2(\bm x)$:
\begin{equation}
\frac{\partial}{\partial\sigma_n^2(\bm x)}\rho_\nu^2(\bm x) = \frac{\sigma_\nu^2}{\left(\sigma_n^2(\bm x) + \sigma_\nu^2\right)^2}
\end{equation}
which is obviously strictly positive $\forall \sigma_\nu^2 > 0$ and $\forall \sigma_n^2(\bm x)$.
\end{proof}

\begin{lemma}\label{lemma:monotonicity}
$\forall \bm x, \bm z \in \mathcal{X}$, $\forall n$, $\hat{I}_n\left(\{\bm x, y\}; \Psi(\bm z)\right)$ as given by \cref{eq:approximated_psi_entropy} and \cref{eq:averaged_post_measurement_entropy} is monotonically increasing in $\hat{H}_n\left[\Psi(\bm z)\right]$, $\rho_n^2(\bm x, \bm z)$ and $\sigma_n^2(\bm x)$.  
\end{lemma}

\begin{proof}
The result follows by combining \cref{lemma:monotonicity_ratio,lemma:monotonicity_rho,lemma:rho_nu_increasing_in_sigma} with the fact that $\hat{H}_n\left[\Psi(\bm z)\right]$ is monotonically decreasing in $\left(\mu_n(\bm z)/\sigma_n(\bm z)\right)^2$ and that $\rho_n^2\rho_\nu^2$ is clearly monotonic in $\rho_\nu^2$. 
\end{proof}

\begin{lemma}\label{lemma:entropy_variation_always_positive}
For any finite $\mu^2_n(\bm z) / \sigma^2_n(\bm z)$, the average entropy variation \cref{eq:rewritten_mutual_info} is non negative for all values of $\rho_n^2(\bm x, \bm z)$ and $\rho_\nu^2(\bm x)$, and it is zero iff $\rho_\nu^2(\bm x)\rho_n^2(\bm x, \bm z) = 0$.
\end{lemma}

\begin{proof}
The result follows immediately from \cref{lemma:monotonicity_rho}, after noticing that for $\rho_\nu^2(\bm x)\rho_n^2(\bm x, \bm z) = 0$ the mutual information \cref{eq:rewritten_mutual_info} is zero and that $\rho_\nu^2(\bm x)\rho_n^2(\bm x, \bm z)$ is never negative.
\end{proof}

\begin{lemma}\label{lemma:delta_h_is_bounded_above}
$\forall n$, $\forall \bm x \in S_n$, $\forall \bm z \in \mathcal{X}$, it holds that:
\begin{equation}
\hat{I}_n\left(\{\bm x, y\}; \Psi(\bm z)\right) \leq \ln(2) \frac{\sigma_n^2(\bm x)}{\sigma_\nu^2}
\end{equation}
\end{lemma}

\begin{proof}
This can be shown directly with the following inequality chain:
\begin{equation}
\begin{split}
\hat{I}_n\left(\{\bm x, y\}; \Psi(\bm z)\right) &\leq \ln(2)\left[1 - \sqrt{\frac{1 - \rho_\nu^2(\bm x)\rho_n^2(\bm x, \bm z)}{1 + c_2\rho_\nu^2(\bm x)\rho_n^2(\bm x, \bm z)}}\right] \\
&\stackrel{c_2 \in (-1, 0)}{\leq} \ln(2)\left(1 - \sqrt{1 - \rho_\nu^2(\bm x)\rho_{n}^2(\bm x, \bm z)}\right) \\
&\stackrel{\rho_n\rho_n^2 \in [0, 1]}{\leq} \ln(2)\left(\rho_\nu^2(\bm x)\rho_{n}^2(\bm x, \bm z)\right)\\\
&\stackrel{\rho_n^2 \in [0, 1]}{\leq} \ln(2)\rho_\nu^2(\bm x) \\
&\leq \ln(2)\frac{\sigma_n^2(\bm x)}{\sigma_\nu^2}
\end{split}
\end{equation}
where the first inequality follows from \cref{lemma:monotonicity_ratio}.
\end{proof}

\begin{lemma}\label{lemma:delta_h_is_bounded_below}
$\forall n$, let $\tilde{\bm x} \in \argmax_{S_n}\sigma_n^2(\bm x)$ and let $M^2 \coloneqq \max_{S_n}\mu_n^2(\bm x)$, and $\tilde{\sigma}^2 \coloneqq \sigma_n^2(\tilde{\bm x})$, then it holds that:
\begin{equation}\label{eq:delta_h_lower_bound}
\hat{I}_n\left(\{\bm x_{n+1}, y_{n+1}\}; \Psi(\bm z_{n+1})\right) \geq b(\tilde{\sigma}^2)
\end{equation}
where $b$ is given by
\begin{equation}\label{eq:b_definition}
b(\eta) \coloneqq \ln(2)\exp\left\{-c_1\frac{M^2}{\eta}\right\}\left[1 - \sqrt{\frac{\sigma_\nu^2}{2c_1\eta + \sigma_\nu^2}}\right].
\end{equation}
\end{lemma}

\begin{proof}
This Lemma is only non-trivial in case the posterior mean is bounded on $S_n$, otherwise, if we admit $|\mu_n(\bm x)| \to \infty$, then we just recover the result that the average information gain is positive.

Now, moving to the proof, as first thing we recall that our algorithm always selects the $\argmax_{\bm x \in S_n}$ of $\hat{I}_n\left(\{\bm x, y\}; \Psi(\bm z)\right)$ as next  parameter to evaluate, meaning that, by construction, $\forall \bm x \in S_n$ and $\forall \bm z \in \mathcal{X}$, it holds that:
\begin{equation}
\hat{I}_n\left(\{\bm x_{n+1}, y_{n+1}\}; \Psi(\bm z_{n+1})\right) \geq \hat{I}_n\left(\{\bm x, y\}; \Psi(\bm z)\right)
\end{equation}
This implies, in particular, that $\hat{I}_n\left(\{\bm x_{n+1}, y_{n+1}\}; \Psi(\bm z_{n+1})\right) \geq \hat{I}_n\left(y(\tilde{\bm x}); \Psi(\tilde{\bm x})\right)$, since, by definition, $\tilde{\bm x} \in S_n$ and is, therefore, always feasible. Writing this condition explicitly, we obtain:
\begin{equation}
\begin{split}
\hat{I}_n\left(\{\bm x_{n+1}, y_{n+1}\}; \Psi(\bm z_{n+1})\right) &\geq \hat{I}_n\left(y(\tilde{\bm x}); \Psi(\tilde{\bm x})\right) \\
&= \ln(2)\left[\exp\left\{-c_1\frac{\mu_n^2(\tilde{\bm x})}{\tilde{\sigma}^2}\right\} - \sqrt{\frac{1 - \rho_\nu^2(\tilde{\bm x})}{1 + c_2\rho_\nu^2(\tilde{\bm x})}}\exp\left\{-c_1\frac{\mu_n^2(\tilde{\bm x})}{\tilde{\sigma}^2}\frac{1}{1 + c_2\rho_\nu^2(\tilde{\bm x})}\right\}\right] \\
&\geq \ln(2)\exp\left\{-c_1\frac{M^2}{\tilde{\sigma}^2}\right\}\left[1 - \sqrt{\frac{1 - \rho_\nu^2(\tilde{\bm x})}{1 + c_2\rho_\nu^2(\tilde{\bm x})}}\right] \\
&= b(\tilde{\sigma}^2)
\end{split}
\end{equation}
where we have used the fact that $c_2 \in (-1, 0)$ and that $\rho_\nu^2(\tilde{\bm x}) \in [0, 1]$.
\end{proof}

\begin{lemma}\label{lemma:b_is_increasing}
The function $b$ defined in \cref{eq:b_definition} is monotonically increasing for positive arguments.
\end{lemma}

\begin{proof}
By looking at the definition of $b$
\begin{equation}
b(\eta) \coloneqq \ln(2)\exp\left\{-c_1\frac{M^2}{\eta}\right\}\left[1 - \sqrt{\frac{\sigma_\nu^2}{2c_1\eta + \sigma_\nu^2}}\right]
\end{equation}
we immediately see that both the exponential factor and the term in square brackets are monotonically increasing with the argument $\eta$, if this is positive, so that $b$ is also monotonically increasing with $\eta > 0$. This can also be shown formally by computing the derivative:
\begin{equation}
\frac{1}{\ln(2)}\frac{\,db(\eta)}{\,d\eta} = \frac{1}{\eta^2} M^2 c_1 e^{-c_1 M^2 / \eta} \left(1 - \sqrt{\frac{\sigma_\nu^2}{2c_1\eta + \sigma_\nu^2}}\right) + \frac{\sigma_\nu^2 c_1 e^{-c_1 M^2 / \eta}}{\sqrt{\frac{\sigma_\nu^2}{2c_1\eta + \sigma_\nu^2}} \left(2c_1\eta + \sigma_\nu^2\right)^2}
\end{equation}
and by noticing that, for positive $\eta$, it is always positive, since $c_1 > 0$.
\end{proof}

\begin{corollary}
From \cref{lemma:b_is_increasing} it follows immediately that, for positive arguments, $b^{-1}$ exists and is also monotonically increasing. In \cref{fig:b_example} we show some examples of the function $b^{-1}$ for $M = 0.5$ and various values of $\sigma_\nu^2$.
\end{corollary}

\begin{figure}
\hfill
\begin{center}
     \includegraphics[height=0.17\textheight]{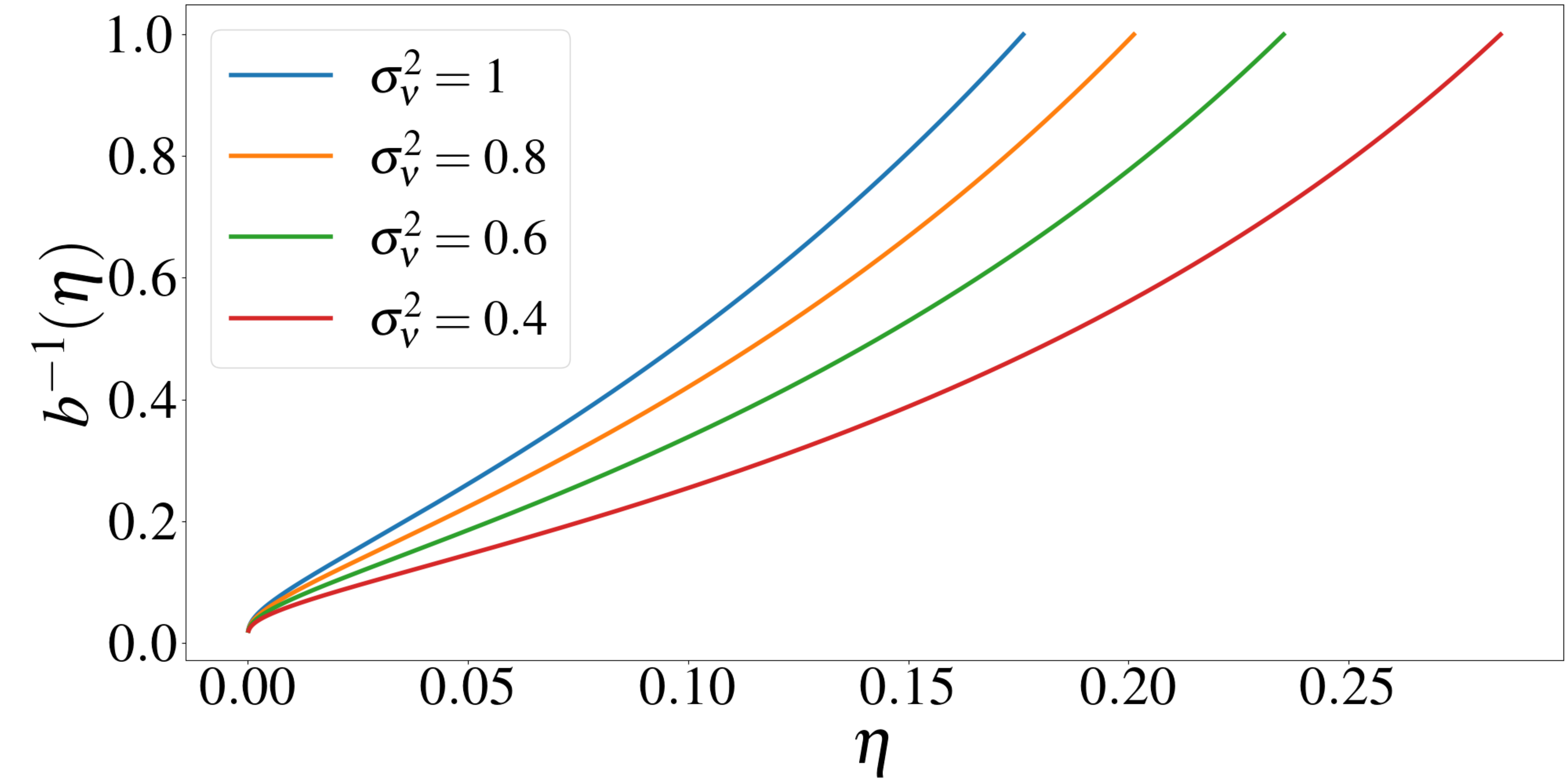}
         \caption{Example plots of the function $b^{-1}$ introduced in \cref{lemma:delta_h_is_bounded_below} for $M=0.5$ and different values of $\sigma_\nu^2$.}
         \label{fig:b_example}
\end{center}
\hfill
\end{figure}

\convergenceboundafterexpansion*
\begin{proof}
In the following $n$ will always be intended $\geq \hat{n}$, where $\hat{n}$ is the one given by the hypothesis. Let us also define again $\tilde{\sigma}_n^2 \coloneqq \max_{S_n}\sigma_n^2(\bm x)$. Finally, let us fix $\varepsilon > 0$.

Combining \cref{lemma:delta_h_is_bounded_above,lemma:delta_h_is_bounded_below}, we obtain:
\begin{equation}
\begin{split}
b(\tilde{\sigma}_n^2) \leq \hat{I}_n\left(\{\bm x_{n+1}, y_{n+1}\}; \Psi(\bm z_{n+1})\right) \leq \ln(2)\frac{\sigma_n^2(\bm x_{n+1})}{\sigma_\nu^2} \\
\Longrightarrow b(\tilde{\sigma}_n^2) \leq \ln(2)\frac{\sigma_n^2(\bm x_{n+1})}{\sigma_\nu^2}
\end{split}
\end{equation}
we can now exploit the monotonicity of $b$ (\cref{lemma:b_is_increasing}) and the fact that $\tilde{\sigma}_n^2$ is not increasing if the safe set does not expand to conclude that:
\begin{equation}
b(\tilde{\sigma}_n^2) \leq b(\tilde{\sigma}_m^2) \leq \ln(2)\frac{\sigma_m^2(x_{m+1})}{\sigma_\nu^2} ~~~~~~ \forall n \geq m \geq \hat{n}
\end{equation}
we can then use this to write:
\begin{equation}\label{eq:b_bounded_by_gamma}
\begin{split}
(n - \hat{n})b(\tilde{\sigma}_n^2) = \sum_{i = \hat{n}}^{n}b(\tilde{\sigma}_n^2) &\leq \\
\ln(2)\sum_{i = \hat{n}}^{n}\sigma_\nu^{-2}\sigma_i^2(\bm x_{i+1}) &\leq \\
\frac{\ln(2)}{\sigma_\nu^2\ln(1 + \sigma_\nu^{-2})}\sum_{i = \hat{n}}^{n}\ln\left(1 + \sigma_\nu^{-2}\sigma_i^2(\bm x_{i+1})\right) &\leq \\
C\gamma_{n - \hat{n}}
\end{split}
\end{equation}
where $\gamma_{n - \hat{n}}$ is the maximum information capacity and $C = \frac{\ln(2)}{\sigma_\nu^2\ln(1 + \sigma_\nu^{-2})}$. The second last passage follows from the fact that $x \leq \ln(1 + x) \sigma_\nu^{-2} / \ln(1 + \sigma_\nu^2)$ for $x \in [0, \sigma_\nu^{-2}]$ together with the fact that $\sigma_\nu^{-2}\sigma_i^2(\bm x_{i+1}) \leq \sigma_\nu^{-2}k(\bm x_{i+1}, \bm x_{i+1}) \leq \sigma_\nu^{-2}$. Finally, the last passage uses the representation of the mutual information $I(\{y_n\}; \{f(\bm x_n)\})$ presented by \citet{srinivas_gaussian_2010}.

Using \cref{eq:b_bounded_by_gamma}, we can show that the minimum $N_\varepsilon$ satisfying the claim of the theorem is the one given by \cref{eq:N_epsilon}:
\begin{equation}\label{eq:N_epsilon_proof}
N_\varepsilon = \min\left\{N \in \mathbb{N} : b^{-1}\left(\frac{C\gamma_N}{N}\right) \leq \varepsilon\right\}
\end{equation}
and we are now able to conclude that, as long as the information capacity grows sub-linearly in $N$, the set on the r.h.s. of~\cref{eq:N_epsilon_proof} is not empty $\forall \varepsilon > 0$. This is guaranteed by the fact that $b^{-1}$ is monotonically increasing, since so is its inverse $b$.
To check that this $N_\varepsilon$ indeed satisfies the claim, one just has to apply $b^{-1}$ on both initial and final state of \cref{eq:b_bounded_by_gamma} and then substitute $\hat{n} + N_\varepsilon$ in the place of $n$; the rest follows from the fact that the maximum variance is non increasing on $S_n$ as long as the safe set does not expand.
\end{proof}

\begin{corollary}\label{corollary:bound_on_information_gain}
Under the hypothesis of \cref{lemma:convergence_bound_after_expansion}, $\forall \varepsilon~>~0$ $\exists N_\varepsilon \in [0, \infty)$ s.t.\@ $\hat{I}_n\left(\{\bm x, y\}; \Psi(\bm z)\right) \leq \varepsilon$ $\forall n \geq \hat{n} + N_\varepsilon$. 
\end{corollary}

\begin{proof}
This follows directly from \cref{lemma:convergence_bound_after_expansion} and from the fact that $\hat{I}_n\left(\{\bm x, y\}; \Psi(\bm z)\right)$ is upper bounded by a monotonic function of the posterior variance (\cref{lemma:delta_h_is_bounded_above}).
\end{proof}

Moving on to the safety guarantees, in \cref{sec:theory} we claimed that any parameter selected according to \cref{eq:x_n_+_1} is safe with high probability. The following result makes this statement precise.

\begin{lemma}
Let $f: \mathcal{X} \rightarrow \mathbb{R}$ have bounded norm in the Reproducing Kernel Hilbert Space $\mathcal{H}_k$ associated to some kernel $k: \mathcal{X} \times \mathcal{X} \rightarrow \mathbb{R}$ with $k(\bm x, \bm x') \leq 1$, and let $S_n$ be the corresponding safe set as defined in \cref{eq:safe_set_definition}, with $\bm x_0$ such that $f(\bm x_0) > 0$. Then, if $\bm x_{n+1}$ is selected according to \cref{eq:x_n_+_1}, it follows that ${P\{f(\bm x_n) \geq 0 \text{ for all } n\} \geq 1 - \delta}$.
\end{lemma}

\begin{proof}
By construction of the sequence $\{\beta_n\}$ we know that ${P\{f(\bm x) \geq \mu_n(\bm x) - \beta_n\sigma_n(\bm x) \text{ for all } n \text{, for all } \bm x\} \geq 1 - \delta}$. The claim then follows by recalling that the acquisition \cref{eq:x_n_+_1} only selects parameters within $S_n$ and that, by construction of of the safe set, $\mu_n(\bm x) - \beta_n\sigma_n(\bm x) \geq 0$ for all $\bm x \in S_n \setminus \{\bm x_0\}$, in addition to the fact that $f(\bm x_0) > 0$ by assumption.
\end{proof} 

Finally, in \cref{sec:theory}, we claimed that, under the assumptions of the theory, the posterior mean $\mu_n(\bm x)$ is bounded by $2\beta_n$ with high probability. The following lemma makes this statement precise. This result formalizes the intuition that for a regular enough GP, in order to get an exploding posterior mean, one needs to be consistently unlucky with the measurement noise.

\begin{lemma}
Let $f$ be a real valued function on $\mathcal{X}$ and let $\mu_n$ and $\sigma_n$ be the posterior mean and standard deviation of a GP($\mu_0, k$) such that it exists a non-decreasing sequence of positive numbers $\{\beta_n\}$ for which $P\left\{f(\bm x) \in [\mu_n(\bm x) \pm \beta_n\sigma_n(\bm x)]~\forall \bm x,~\forall n\right\} \geq 1 - \delta$. Moreover, assume that $\mu_0(\bm x) = 0$ for all $\bm x$ and that $k(\bm x, \bm x^\prime) \leq 1$ for all $\bm x, \bm x^\prime \in \mathcal{X}$. Then it follows that $|\mu_n(\bm x)| \leq 2\beta_n$ with probability of at least $1 - \delta$ jointly for all $\bm x$ and for all $n$.
\end{lemma}

\begin{proof}
From the hypothesis, it follows that the following two conditions hold for all $\bm x$ and for all $n$ with probability of at least $1 - \delta$:
\begin{equation}
|f(\bm x)| \in \left[0, \beta_0\sigma_0(\bm x)\right]
\end{equation}
\begin{equation}
\mu_n(\bm x) \in \left[-|f(\bm x)| - \beta_n\sigma_n(\bm x), |f(\bm x)| + \beta_n\sigma(\bm x) \right]
\end{equation}
From these two conditions, it follows that $|\mu_n(\bm x)| \leq \beta_0\sigma_0(\bm x) + \beta_n\sigma_n(\bm x)$ with probability of at least $1 - \delta$. Now, we recall that the sequence $\{\beta_n\}$ is non decreasing by assumption and that the sequence $\{\sigma_n(\bm x)\}$ is non increasing by the properties of a GP, which allows us to conclude that $|\mu_n(\bm x)| \leq 2\beta_n\sigma_0(\bm x)$, which concludes the proof once we recall the assumption that $k(\bm x, \bm x^\prime) \leq 1$ for all $\bm x, \bm x^\prime \in \mathcal{X}$. The result can easily be extended to the case of non zero prior mean, by just adding the prior mean as offset in the found upper bound for the posterior mean.
\end{proof}

\section{Entropy of \texorpdfstring{$\Psi(\bm{x})$}{\b} approximation}\label{appendix_entropy_approx}

In order to analytically compute the mutual information $\hat{I}_n\left(\{\bm x, y\}; \Psi(\bm z)\right) = H_n\left[\Psi(\bm z)\right] - \mathbb{E}_{y}\left[H_{n+1}\left[\Psi(\bm z) \middle| \{\bm x, y\}\right]\right]$, we have approximated the entropy of the variable $\Psi(\bm x)$ at iteration $n$ with $\hat{H}_n\left[\Psi(\bm x)\right]$, given by \cref{eq:approximated_psi_entropy}, which we have then used to derive the results presented in the paper. The approximation allowed us to derive a closed expression for the average of the entropy at parameter $\bm z$ after an evaluation at $\bm x$, $\mathbb{E}_{y}\left[\hat{H}_{n+1}\left[\Psi(\bm z) \middle| \{\bm x, y\}\right]\right]$. 
This approximation was obtained by noticing that the exact entropy \cref{eq:exact_psi_entropy} has a zero mean Gaussian shape, when plotted as function of $\mu_n(\bm x) / \sigma_n(\bm x)$, and then by expanding both the exact expression \cref{eq:exact_psi_entropy} and a  zero mean unnormalized Gaussian in $\mu_n(\bm x) / \sigma_n(\bm x)$ in their Taylor series around zero. At the second order we obtain, respectively,
\begin{equation}\label{eq:entropy_first_order}
H_n\left[\Psi(\bm x)\right] = \ln(2) - \frac{1}{\pi}\left(\frac{\mu_n(\bm x)}{\sigma_n(\bm x)}\right)^2 + o\left(\left(\frac{\mu_n(\bm{x})}{\sigma_n(\bm{x})}\right)^2\right)
\end{equation}
and 
\begin{equation}\label{eq:gaussian_first_order}
c_0 \exp\left\{-\frac{1}{2\sigma^2}\left(\frac{\mu_n(\bm{x})}{\sigma_n(\bm{x})}\right)^2\right\} = c_0 - c_0\frac{1}{2\sigma^2}\left(\frac{\mu_n(\bm{x})}{\sigma_n(\bm{x})}\right)^2 + o\left(\left(\frac{\mu_n(\bm{x})}{\sigma_n(\bm{x})}\right)^2\right).
\end{equation}
By equating the terms in \cref{eq:entropy_first_order} with the ones in \cref{eq:gaussian_first_order}, we find ${c_0 = \ln(2)}$ and ${\sigma^2 = \ln(2)\pi/2}$, which leads to the approximation for $H_n\left[\Psi(\bm x)\right]$ \cref{eq:approximated_psi_entropy} used in the paper: ${H_n\left[\Psi(\bm x)\right] \approx \ln(2) \exp\left\{-\frac{1}{\pi\ln(2)}\left(\frac{\mu_n(\bm x)}{\sigma_n(\bm x)}\right)^2\right\}}$.
In \cref{fig:psi_entropy_comparison} we plot \cref{eq:exact_psi_entropy} and \cref{eq:approximated_psi_entropy} against each other as a function of the mean-standard deviation ratio, while \cref{fig:psi_entropies_diff} shows the difference between the two. From these two plots, one can see almost perfect agreement between the two functions, with a non negligible difference limited to two small neighborhoods of the $\mu/\sigma$ space. 

\begin{figure}[h]
\hfill
     \centering
     \begin{subfigure}[b]{0.49\textwidth}
         \centering
         \includegraphics[height=0.15\textheight]{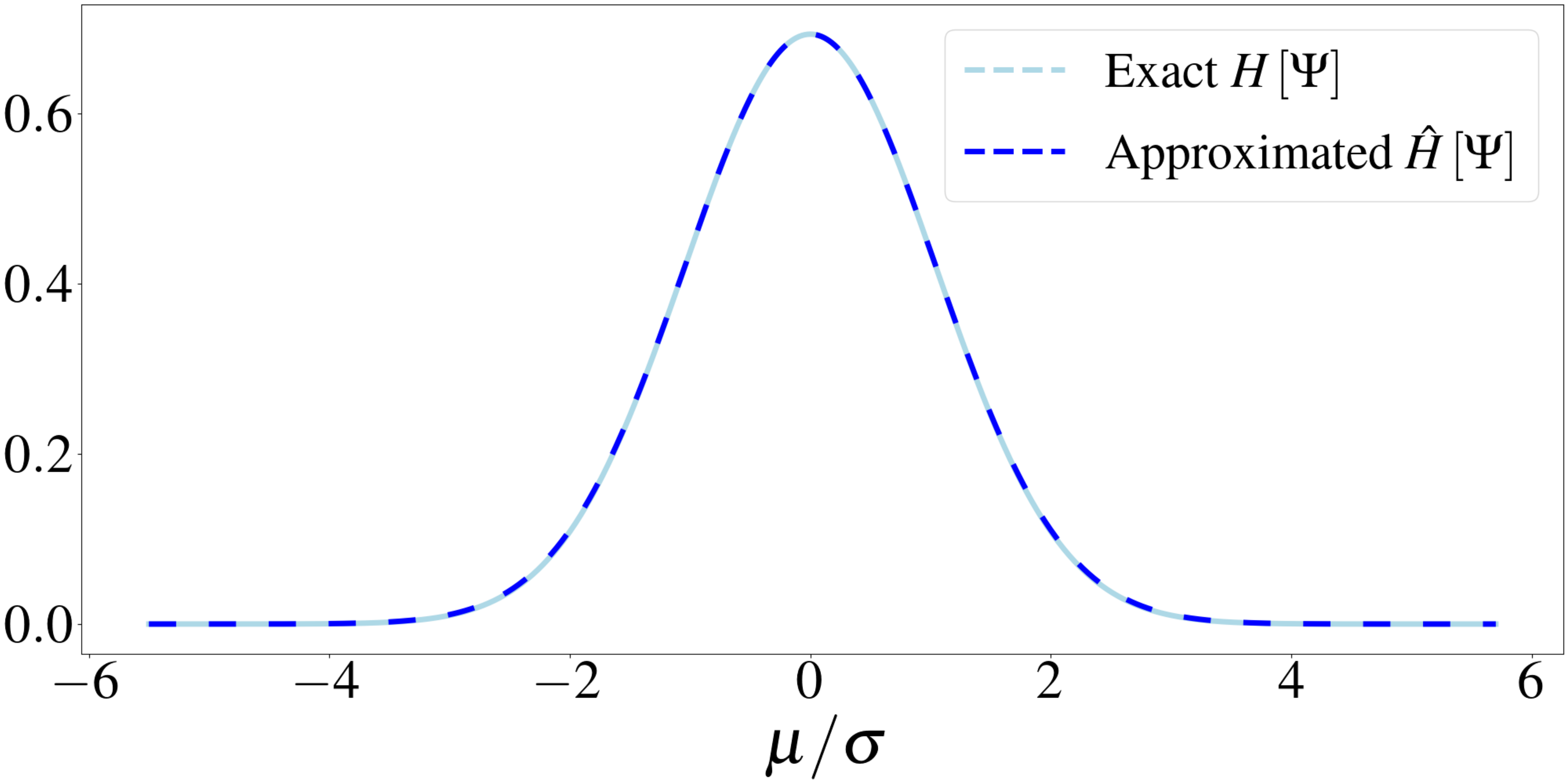}
         \caption{Exact and approximated $H[\Psi(\bm x)]$.}
         \label{fig:psi_entropy_comparison}
     \end{subfigure}
     \hfill
     \begin{subfigure}[b]{0.49\textwidth}
         \centering
         \includegraphics[height=0.15\textheight]{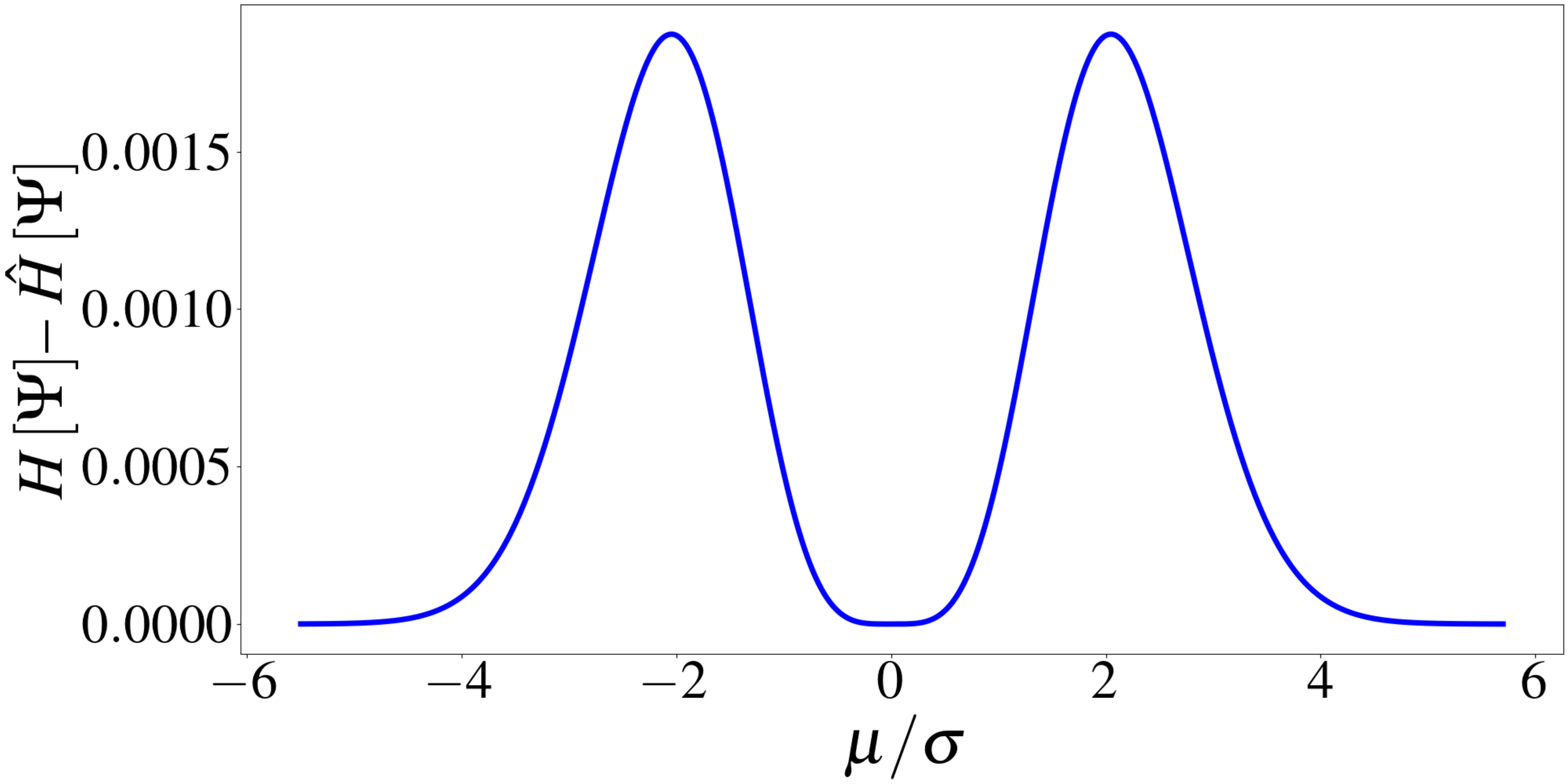}
         \caption{Approximation error.}
         \label{fig:psi_entropies_diff}
     \end{subfigure}
        \caption{Comparison between exact entropy of $\Psi(\bm x)$ \cref{eq:exact_psi_entropy} and approximated form \cref{eq:approximated_psi_entropy}: (a) the two entropies plotted against each other; (b) the approximation error expressed as difference of the two.}
        \label{fig:psi_entropies}
        \hfill
\end{figure}

\section{Experiments Details}\label{appendix_experiments}

In this Appendix, we collect details about the experiment presented in \cref{experiments}. Code for the used acquisition functions can be found at \url{https://github.com/boschresearch/information-theoretic-safe-exploration}.

\ourmethod selects the next parameter to evaluate according to \cref{eq:x_n_+_1}, which is a non convex optimization problem constrained in one of the variables. We find the solution to this problem via constrained gradient ascent with multiple restarts. 

\paragraph{GP samples}
For the results shown in \cref{fig:comparison} we run both \ourmethod and the expansion stage of \stageopt on 100 samples from a GP defined on the square $[-2.5, 2.5] \times [-2.5, 2.5]$, with RBF kernel with the following hyperparameters: $\mu_0 \equiv 0$; kernel lengthscale = $0.1$; kernel outputscale = 150; $\sigma_\nu^2 = 0.05$, while the safe seed $\bm x_0$ was chosen as the origin: $\bm x_0 = (0, 0)$. For the \stageopt runs, we used the code by \citet{berkenkamp_bayesian_2020}, who open-sourced it on GitHub under the MIT license (\url{https://github.com/befelix/SafeOpt}). As \stageopt requires a discretized domain, we used the same uniform discretization of 700 points per dimension for all GP samples. Finally, the percentage of the domain classified as safe is estimated via Monte Carlo sampling.
Concerning the safety violations summarized in \cref{tab:safety_violations}, the fact that they are comparable is expected, since in our experiments they all use the posterior GP confidence intervals to define the safe set. 

\begin{table}[hb]
  \begin{center}
    \caption{Average percentage of safety violations per run over the 100 runs used to obtain \cref{fig:comparison}.}
    \label{tab:safety_violations}
    \begin{tabular}{c c c c c c}
    \specialrule{.1em}{.05em}{.05em} 
    \rule{0mm}{4mm}
       & \ourmethod & SO L=0 & SO L=1 & SO L=5 & SO L=10\\ 
      \hline
      \% of safety violations & 0.04 $\pm$ 0.20 & 0.01 $\pm$ 0.12 & 0.03 $\pm$ 0.19 & 0.05 $\pm$ 0.22 & 0.05 $\pm$ 0.25\\
      \specialrule{.1em}{.05em}{.05em} 
    \end{tabular}
  \end{center}
\end{table}

To evaluate whether or not \ourmethod converges to the same safe set as \stageopt-like exploration does, we performed the same experiment as for \cref{fig:comparison}, but with the difference that for this experiment we used a bigger kernel lengthscale of 1.6. For each GP sample, the true reachable safe set is obtained by sampling according to the rule $\bm x_{n+1} \in \argmax_{S_n}\sigma_n^2(\bm x)$, starting from $\bm x_0$, until the uncertainty over the safe set $S_n$ was reduced under the noise variance. We show the results in \cref{fig:convergence}, which shows that, indeed, both \ourmethod and \stageopt-like exploration lead to the discovery of the same largest safe set. Similarly as for \cref{fig:comparison}, the percentages that we show are then obtained via Monte Carlo sampling. 
For the plot in \cref{fig:one_d_exploration}, the constraint function was $f(x) = e^{-x} + 0.05$ and we used a RBF kernel with hyperparameters: $\mu_0 \equiv 0$; kernel lengthscale = $1.2$; kernel outputscale = 100; $\sigma_\nu^2 = 0.05$, while the safe seed was $\bm x_0 = 0$, and the domain for the \stageopt exploration was composed of 500 points. The chosen function is a slightly offset exponential. On one side of the domain this constraint function becomes increasingly close to the safety threshold, making it hard to explore with high Lipschitz constant. On the other hand, if the Lipschitz constant is too small, the algorithm will prefer to reduce uncertainty away from the border. On the contrary, \ourmethod will always tend to select parameters close to the boundary. This intuition is what justifies the results shown in \cref{fig:one_d_exploration}. In \cref{fig:comparison_with_std_deviation} we report the same plots as in \cref{fig:safe_set_expansion_comparison}, but with error bars representing the standard deviation instead of the standard error.

\begin{figure}
	\centering\includegraphics[height=0.20\textheight]{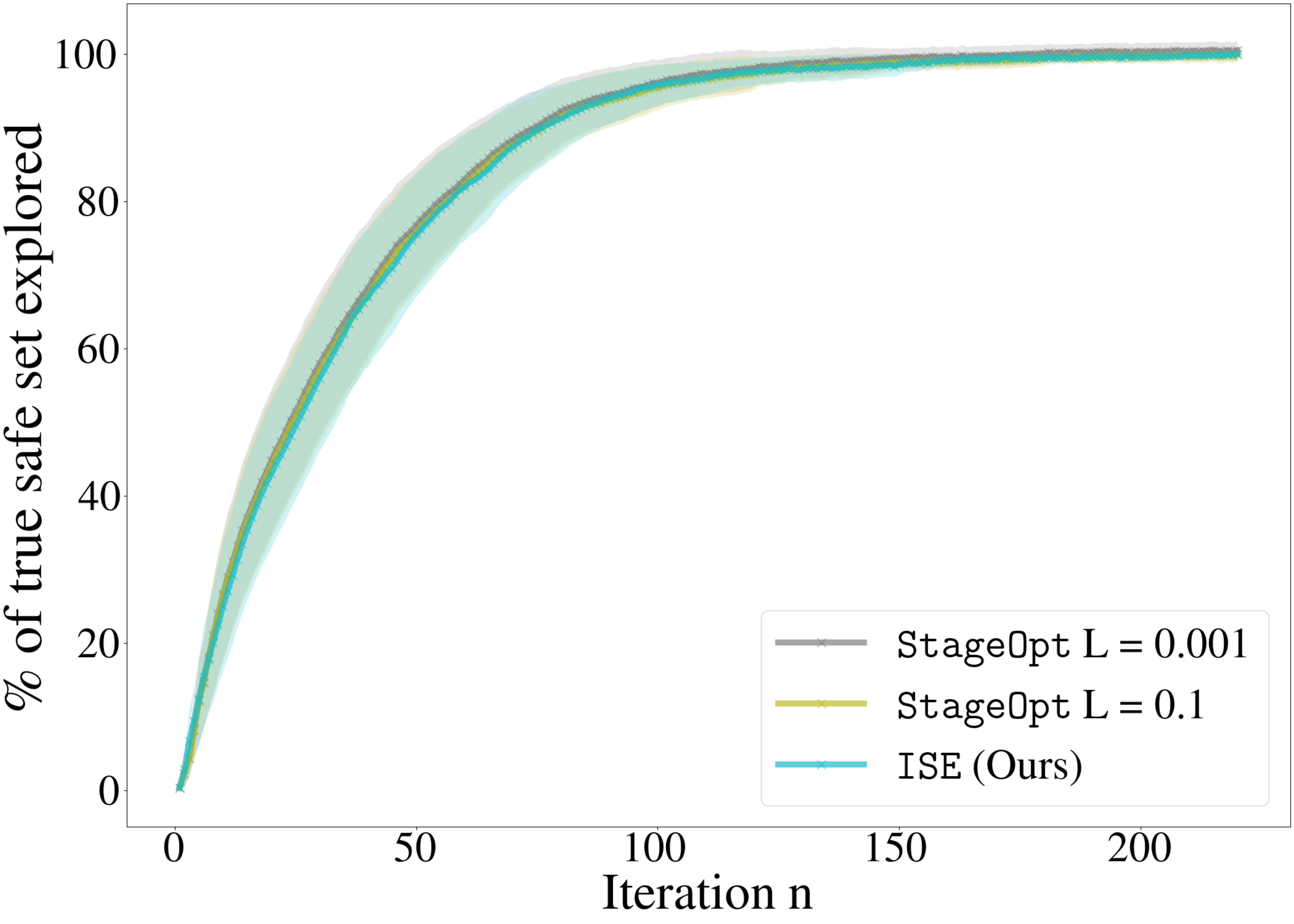}
	\caption{The percentage of the maximally reachable safe set classified as safe is plotted as function of $n$: we can see that \ourmethod eventually leads to discover the same reachable safe set discovered by \stageopt-like exploration.}
   \label{fig:convergence}
\end{figure}

\begin{figure*}
\hfill
     \centering
     \begin{subfigure}{0.495\textwidth}
         \centering
         \includegraphics[width=\textwidth,height=0.15\textheight]{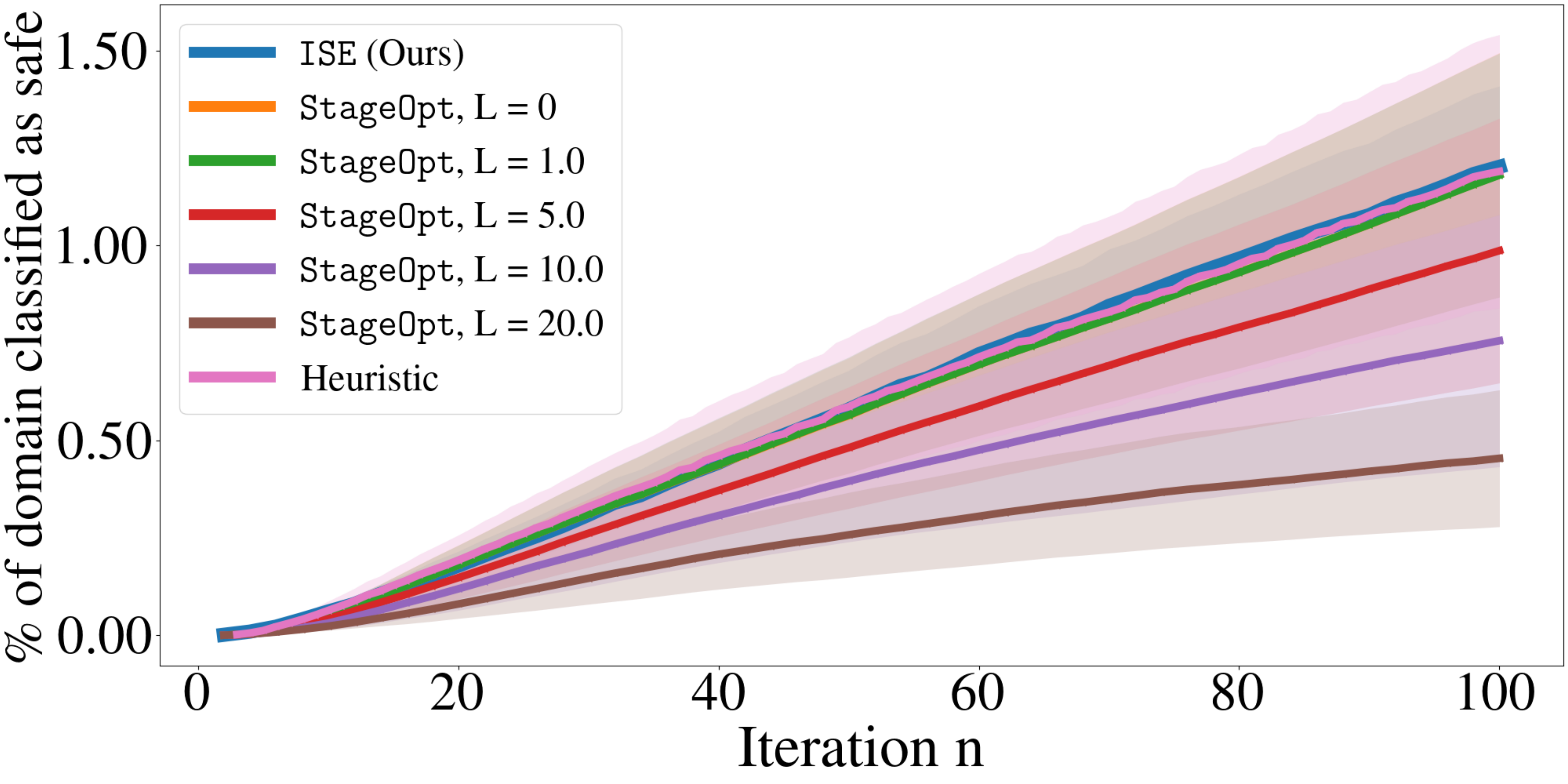}
         \caption{Comparison with \stageopt for GP samples.}
         \label{fig:comparison_std_Dev}
     \end{subfigure}
     \begin{subfigure}{0.495\textwidth}
         \centering
         \includegraphics[width=\textwidth,height=0.15\textheight]{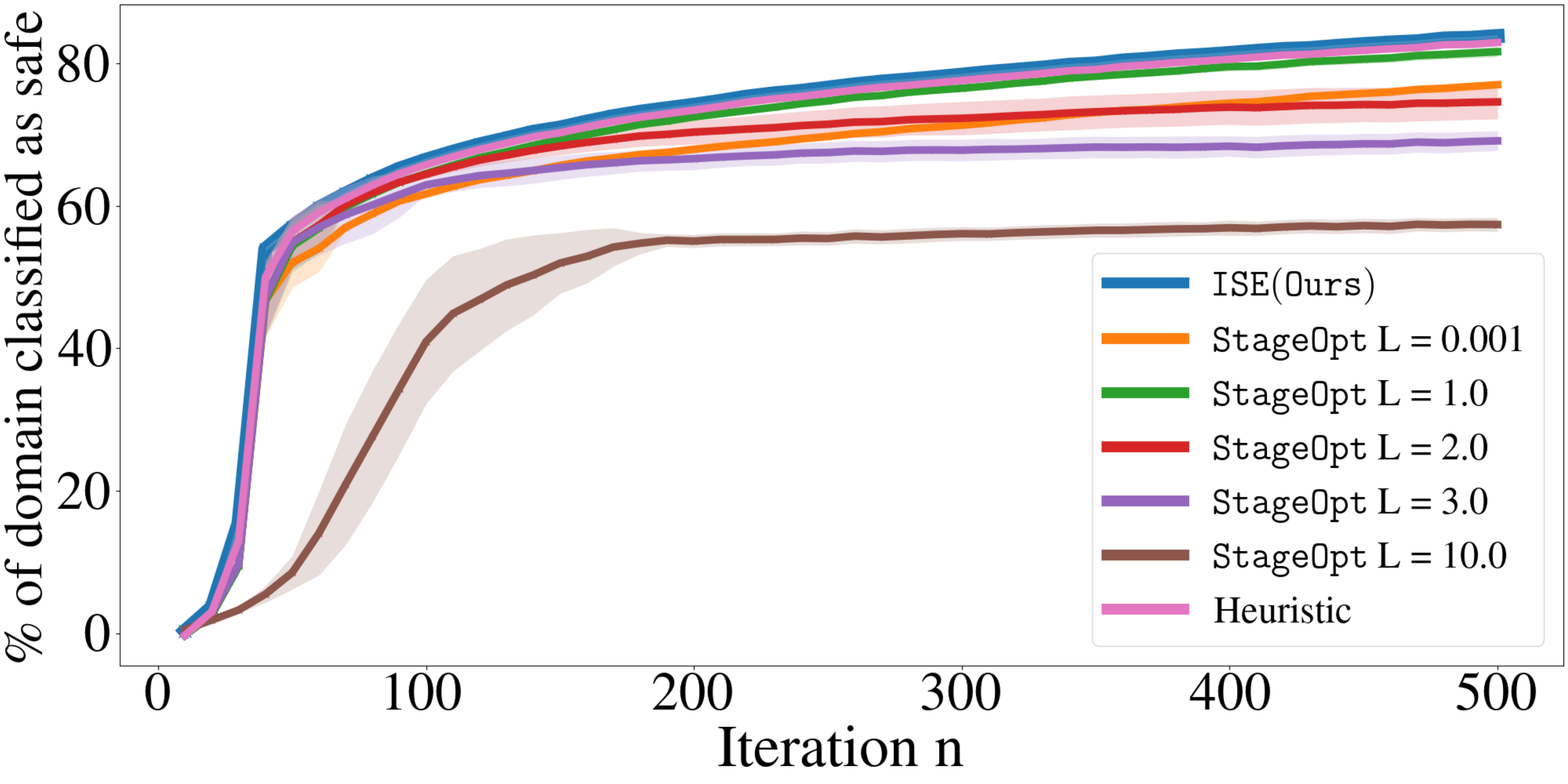}
         \caption{Comparison in 1D exponential example.}
         \label{fig:one_d_exploration_std_dev}
     \end{subfigure}
        \caption{In these plots we show the same results as in \cref{fig:safe_set_expansion_comparison}, but here we plot the standard deviation instead of the standard error.}
        \label{fig:comparison_with_std_deviation}
\end{figure*}

\paragraph{OpenAI Gym Control}
For the inverted pendulum and cart pole experiments, we used the environments provided by the OpenAI gym \citep{brockman2016openai} under the MIT license. The cart pole environment by default accepts only discrete actions $u_t \in \{0, 1\}$, causing a push of fixed strength either to the left or right. Instead of mapping the output of our linear controller $u_t = \alpha_1 \theta_t + \alpha_2 \dot{\theta_t} + \alpha_3 \dot{s_t}$ to $\{0, 1\}$, we modified the environment to accept continuous actions, corresponding to pushes of varying intensity in the direction specified by the action's sign. 
For the inverted pendulum experiment, the threshold angular velocity $\dot{\theta}_M$ was set to $0.5$ rad$/s$, with an episode length of 400 steps, and \cref{fig:pendulum_experiment} shows one run of \ourmethod using a GP with RBF kernel with the following hyperparameters: $\mu_0 \equiv 0$; kernel lengthscale = $1.3$; kernel outputscale = 6.6; $\sigma_\nu^2 = 0.04$. For the cart pole one, the episode length was set to 200 steps and the threshold angle $\theta_M$ was of $0.28$ radians. The GP we used in this case had a RBF kernels with hyperparameters: $\mu_0 \equiv 0$; kernel lengthscale = $0.8$; kernel outputscale = 5; $\sigma_\nu^2 = 0.05$. The safe seed $\bm \alpha_0$ was set to $\bm \alpha_0 = (-0.0073, -1.39, 2.01)$, while the domain was set to $[-2, 0] \times [-2, 1.5] \times [-2, 7]$. The average percentage of true safe set classified as safe plotted in \cref{fig:cartpole_safe_set} is over 100 runs and is estimated via Monte Carlo sampling. For what concerns the comparison about the number of unsafe evaluations in the cart pole task, the average percentage of safety violations was of $5.02 \pm 0.95$ for the \stageopt runs, while for \ourmethod it was of $5.5 \pm 0.98$.

\paragraph{High dimensional domains}
For the five dimensional experiment we used the same custom \linebo wrapper for both the \ourmethod and \stageopt acquisitions, which at each iteration randomly selects multiple one-dimensional subspaces and then finds the optimum of the respective acquisition function restricted to these subspaces. In these experiments we used a GP with RBF kernel with hyperparameters: $\mu_0 \equiv 0$; kernel lengthscale = $1.6$; kernel outputscale = 1; while the safe seed $\bm x_0$ was set to $\bm x_0 = (-0.2)^d$ and the observation noise to $\sigma_\nu^2 = 0.5$.

\paragraph{Heteroskedastic noise domains}
In these experiments we used the same \linebo wrapper as in the five dimensional experiment. The GP had a RBF kernel with hyperparameters: $\mu_0 \equiv 0$; kernel lengthscale = $1.6$; kernel outputscale = 1; while the safe seed $\bm x_0$ was set to the origin. For what concerns the observation noise, as explained in \cref{par:high_heteroskedastic}, we used heteroskedastic noise, with two different values of the noise variance in the two symmetric halves of the domain. In particular, given a parameter $\bm x = (x_1, x_2, \dots, x_n)$, we set the noise variance to $\sigma_\nu^2 = 0.05$ if $x_0 \geq 0$, otherwise we set it to $\sigma_\nu^2 = 0.5$. As explained in \cref{experiments}, the constraint function is $f(\bm x) = \frac{1}{2}e^{-\bm x^2} + e^{-(\bm x \pm \bm x_1)^2} + 3e^{-(\bm x \pm \bm x_2)^2} + 0.2$, with $\bm x_1$ and $\bm x_2$ given by: $\bm x_1 = (2.7, 0, \dots, 0)$ and $\bm x_2 = (6, 0, \dots, 0)$.

\paragraph{Computational resources}
The experiments were run on a HPC cluster, with each experiment using four Intel Xeon Gold CPUs. All experiments (including early evaluations) amounted to a total of 77020 hours. The Bosch Group is carbon neutral. Administration, manufacturing and research activities do no longer leave a carbon footprint. This also includes GPU clusters on which the experiments have been performed.

\end{document}